%% file: neurips_2026.tex
\documentclass{article}

\usepackage[preprint]{neurips_2026}

\IfFileExists{NeurIPS26/preamble.tex}{\input{NeurIPS26/preamble}}{\input{preamble}}

\usepackage[utf8]{inputenc} %
\usepackage[T1]{fontenc}    %
\usepackage{hyperref}       %
\usepackage{url}            %
\usepackage{booktabs}       %
\usepackage{amsfonts}       %
\usepackage{nicefrac}       %
\usepackage{microtype}      %
\usepackage{xcolor}         %
\usepackage{xspace}

\hypersetup{
    colorlinks,
    linkcolor={blue!50!black},
    citecolor={blue!50!black},
    urlcolor={blue!80!black}
}

\newcommand{\indep}{\mathrel{\perp\mkern-10mu\perp}}
\newcommand{\floor}[1]{\lfloor#1\rfloor}
\newcommand{\dell}[1]{\left(#1\right)}

\newcommand{\setl}[1]{\left\{#1\right\}}
\newcommand{\setcl}[2]{\left\{#1\middle|#2\right\}}
\newcommand{\sbrl}[1]{\left[#1\right]}
\newcommand{\sbrcl}[2]{\left[#1\middle|#2\right]}
\newcommand{\absl}[1]{\left|#1\right|}

\DeclareBoldMathCommand{\c}{c}

\newcommand{\mubar}{\bar{\mu}}

\newcommand{\Deltamin}{\Delta_\text{min}}

\newcommand{\logdelta}{\log \tfrac{1}{\delta}}
\newcommand{\logdeltaK}{\log \tfrac{K - 1}{\delta}}
\newcommand{\logdeltadash}{\log \tfrac{1}{\delta'}}

\newcommand{\glrtz}{Z^\text{GLRT}}

\newcommand{\algname}{\textsf{ISM}}

\newcommand{\Ncal}{\mathcal{N}}

\makeatletter
\newcommand{\vast}{\bBigg@{4}}
\newcommand{\Vast}{\bBigg@{5}}
\makeatother

\newcommand{\ups}[1]{^{(#1)}}

\newcommand{\astar}{a^*}
\newcommand{\atilde}{\tilde{a}}

\newcommand{\lband}{\mathcal B_{\bm \Delta}}
\newcommand{\lalt}{\mathcal B_{\tilde{\bm \Delta}}}

\usepackage{stackrel}

\usepackage{todonotes}

\renewcommand{\log}{\ln}

\let\ex\E
\DeclareMathOperator{\var}{\mathbb V}

\usepackage{times}

\title{Best Arm Identification for Bandits with Shifting Means}

\author{%
  Lukas Zierahn\\
  CWI and Booking.com\\
  Science Park 123 1098 XG Amsterdam\\
  \texttt{lukas.zierahn@gmail.com} \\
  \And
  Wouter M. Koolen\\
  CWI and University of Twente\\
  Science Park 123 1098 XG Amsterdam\\
  \texttt{wmkoolen@cwi.nl} \\
  \AND
  Shubhada Agrawal \\
  Indian Institute of Science (IISc), Bengaluru\\
  CV Raman Rd, Bengaluru, Karnataka 560012, India\\
  \texttt{shubhada@iisc.ac.in}\\
  \And
  Christina Katsimerou \\
  Booking.com \\
  Oosterdokskade 163, 1011 DL Amsterdam\\
  \texttt{christina.katsimerou@booking.com} \\
  \And
  Dirk {van der Hoeven} \\
  Leiden University \\
  Rapenburg 70, 2311 EZ Leiden\\
  \texttt{dirk@dirkvanderhoeven.com} \\
}

\begin{document}

\maketitle

\begin{abstract}
  We study the best arm identification problem in a stochastic environment with a novel form of adversarial perturbations, which we coin Shifting Means. While classically the mean rewards of the $K$ arms are stable in time, in Shifting Means only the gaps $\boldsymbol{\Delta}$ between mean rewards are stable, while their common shift may be determined adversarially in each round. The objective of the learner is to identify the best arm with high probability while minimizing sample complexity (the fixed confidence setting). Handling shifts requires new tools: we show that algorithms employing a Generalized Likelihood Ratio Test (GLRT) stopping rule, including the popular Track-and-Stop, fail under time-varying shifts. Instead, we propose Importance Weights for Shifting Means ($\mathsf{ISM}$). Assuming means bounded by $U$ and $\sigma^2$-sub-Gaussian rewards, we show $\mathsf{ISM}$ to be $\delta$-correct and to enjoy a sample complexity bound of order $K (\sigma^2 + U^2) \Delta_{\min}^{-2} \ln \frac{1}{\delta}$. We also present a matching (up to constant factors) worst-case lower bound and evaluate our results empirically.
\end{abstract}

\section{Introduction}
Best Arm Identification (BAI) is a fundamental problem in sequential decision making, where the learner is given a set of $K$ arms, each associated with an unknown probability distribution. At each time step, the learner may sample one arm and observe a reward, with the goal of identifying the best arm, namely, the one with the highest expected reward. 

In the literature, two classical formulations of the problem have been considered. In the fixed-budget setting, the objective is to maximize the confidence of identifying the best arm given a pre-specified sampling budget \citep{glynn2004large,audibert2010best}. In the fixed-confidence setting, the objective is to identify the best arm with a prescribed level of confidence while minimizing the number of samples \citep{evendar06SH,mannor2004sample}. This work focuses exclusively on the fixed-confidence setting.

A central assumption in the standard BAI framework is that observations are independent and identically distributed (i.i.d.) over time. However, in real-world applications, data are neither independent nor identically distributed. For example, the performance of a new product may decline as novelty wears off, or the conversion rate of a website may fluctuate with time of day or day of the week. When such perturbations are left unmodeled, classical algorithms can lose their correctness guarantees (Section~\ref{sec:TASFails}). To provide robustness against such effects, we study BAI under adversarial perturbations that shift the mean rewards of all arms simultaneously.

One important class of such perturbations arises when global trends affect all arms in the same direction. For instance, user engagement on an online platform may be uniformly higher during weekends and lower during weekdays, or all products may experience a dip in sales following a seasonal lull. Empirical evidence suggests that macro-level effects on real-world platforms (e.g., seasonality, events, platform-wide fluctuations) are largely independent of the intrinsic qualities of the arms and that the differences between arms remain stable over time \citep{fiez2024best,subpopulations}. We model such perturbations adversarially: the absolute performance of all arms is shifted by an arbitrary, unobserved common offset in each round, while the pairwise gaps remain constant. This ensures that a unique best arm persists throughout the experiment, while the adversarial shift captures a broad class of real-world effects without assuming any specific generative process. We call the resulting model Shifting Means.

In Shifting Means, the expected rewards may vary across rounds and allow an arbitrary dependence on past samples. To retain learnability we introduce a latent variable $B_t$ called the \textit{shift}, which the learner will never observe. $B_t$ can vary arbitrarily and even adversarially between timesteps $t$ and may also be drawn at random. We make three core assumptions: (a) the expected mean of an arm given the shift $B_t$ is bounded by $U$, (b) the reward distributions of all arms are $\sigma^2$ sub-Gaussian given $B_t$, and (c) the expected mean gaps between arms given $B_t$ are constant (both over the draw of $B_t$ and across time). One valid way to construct an environment for $K$ arms given these assumptions is to pick a set of gaps $\bm \Delta$, pick a bounded $B_t$ deterministically in a non-oblivious adversarial fashion and draw the rewards from normal distributions with means $B_t - \Delta_a$ for arm $a$. Picking $B_t = 0$ with probability 1 recovers the stationary best arm identification setting.

\paragraph{Contributions} We now summarize key contributions of our work. 
\begin{enumerate}
    \item We show that the much celebrated Track-and-Stop algorithm of \citet{garivier16optimalBAI} fails in the Shifting Means setting (Lemma~\ref{lem:glrt}) and propose \underline{I}mportance Weighting for \underline{S}hifting \underline{M}eans (\algname; Algorithm~\ref{alg:ism}). \algname\ is based on an evidence measure that generalizes importance weighting and optimally tunes the hyperparameters to ensure that the evidence of sub-optimal arms being the best arm does not grow and consequently we prove $\delta$-correctness for \algname\ (Theorem~\ref{lem:fullsafety}).
    
    \item We show an asymptotic in $\delta \rightarrow 0$ upper bound on the expected sample complexity in terms of a problem dependent rate $1/D^*(\bm \Delta)$ (Theorem~\ref{thm:sample complexity}), which we then show to be of order $K (\sigma^2 + U^2) / \Deltamin^2$ (Theorem~\ref{thm:ESbound}), for the smallest gap $\Deltamin$. One of the core difficulties to overcome is to ensure quick enough convergence to optimal arm sampling probabilities for the current bandit, which we tackle by per-round plug-in estimates, sufficient exploration, and a martingale concentration analysis of Bernstein type.

    \item We prove a lower bound on the expected number of samples any $\delta$-correct algorithm needs to generate in the Shifting Means scenario, by randomizing the shifts $B_t$. We draw the shifts from the raised cosine distribution, which is least informative in the sense of Fisher information among all distributions on a bounded interval \citep[Equation~(29.13) and references therein]{polyanskiy2025information}. It uses its bounded range to inflate the effective variance of the rewards, measured through a KL divergence, from $\sigma^2$ to nearly $\sigma^2 + U^2/\pi^2$ (Lemma~\ref{lem:kl inflate}), culminating in a lower bound of order $K (\sigma^2 + U^2) / \Deltamin^2$ (Theorem~\ref{thm:lower bound}). That lower bound contains the standard BAI lower bound of $K \sigma^2 / \Deltamin^2$ \citep{kaufmann16complexity} and shows that \algname\ is optimal up to constant factors.
    
    \item We demonstrate the practicality of \algname\ through experiments in Section~\ref{sec:numerics}. We discuss when the theoretical bounds are good predictors of the actual stopping time. We also verify empirically that the algorithm is $\delta$-correct under mis-specification.

\end{enumerate} 

\paragraph{Limitations} Our model assumes an additive common shift and fixed pairwise gaps and we require valid known upper bounds on the conditional means and sub-Gaussian variance proxy. The current expected sample-complexity upper bound is asymptotic as $\delta \to 0$, the upper and lower bounds match only up to universal constants, and the numerical evaluation is limited to synthetic instances.

\paragraph{Related Work}  The first tight lower bound and an asymptotically optimal algorithm for the fixed-confidence BAI setting with stationary parametric distributions was developed by \citet{garivier16optimalBAI}. Their algorithm, Track-and-Stop (TAS), has since been extended to a variety of BAI frameworks, including BAI with heavy-tails and risk-sensitive objectives \citep{agrawal20a,agrawal2021optimal}, BAI with multiple-correct answers \citep{degenne2019pure}, BAI for linear bandits \citep{jedra2020linearBAI}, multi-agent bandits \citep{andreas2023BAIMAMAB}, and batched BAI \citep{jin2024optimal}. For a broader introduction to BAI and the TAS methodology, we refer the reader to the textbook treatment by \citet{lattimore2020banditalgs}.

Robustness to adversarial perturbations is well studied in the regret-minimization literature \citep{lykouris2018stochastic,gupta2019better}. Within BAI, \citet{altschuler2019contaminated} study $\varepsilon$-contamination, where each sample is independently replaced by an adversarial value with probability $\varepsilon$; since means become unidentifiable, they target the arm with the highest median and give matching sample-complexity bounds. \citet{jamieson16nonstochasticBAI} consider a fully adversarial model where arm losses may change arbitrarily but are assumed to converge to a limit, a condition that need not hold in our setting. Closest to our setting, \citet{globalshifts2024} study fixed-budget BAI under a common additive shift $s_j$ that equally affects the mean of all arms in each environment $j$. They show that pairwise gap estimates are consistent under the condition that the number of environment changes grows sublinearly in the number of rounds and that the learner observes when the environment changes. In contrast, Shifting Means allows the shift to change adversarially every round and the learner does not observe any additional information about the shifts.

A related field is non-stationary bandits, since adversarial common shifts can naturally arise from real-world non-stationarity. In the fixed-budget setting, \citet{abbasi2018best} achieve a ``best of both worlds'' guarantee adapting to both stochastic and adversarial rewards, and \citet{xiong24testingrobust} address linear BAI under arbitrary time variations but can be overly conservative for the structured perturbations we consider. \citet{allesiardo2017nonstationary} study a non-stationary BAI setting similar to but more restrictive than ours using elimination-based methods. In the fixed-confidence setting, \citet{subpopulations} model non-stationarity through seasonal subpopulations with known frequencies, achieving asymptotic optimality but not capturing continuous perturbations. More recently, \citet{NEURIPS2024_minimax} generalize this to piecewise-stationary linear bandits where the learner estimates both rewards and change-points.

\paragraph{Outline of the paper} In Section~\ref{sec:setup}, we formalize the setup and introduce the notation we use throughout. Section~\ref{sec:TASFails} contains certain negative results and also give some intuitions on why TAS cannot be used directly for Shifting Means. In Section~\ref{sec:alg} we first introduce the core ideas of our \algname\ algorithm and show its $\delta$-correctness before also demonstrating the sample complexity guarantees. Then we move on to the lower bounds in Section~\ref{sec:lower bounds}, present the numerical results in Section~\ref{sec:numerics}, and conclude with discussions and broader perspectives in Section~\ref{sec:conclusion}. %

\section{Setup and Notation for Shifting Means}\label{sec:setup}
In this section we define our stochastic bandit model with common adversarial perturbations, describe learner strategies in the fixed-confidence setting, and preview the sample complexity guarantees we will achieve.
Let $K$ be the number of arms. We will model sequential interactions where in each round $t \ge 1$ the learner picks an arm $A_t \in [K]$ and the environment responds with a reward $X_t \in \reals$. We denote the history observed in the first $t$ rounds by $H_t = (A_1, X_1, \ldots, A_t, X_t)$. A \emph{sampling rule} for the learner is a stochastic kernel for $A_t | H_{t-1}$ in each round $t \ge 1$, whereas a \emph{bandit instance} $\mathcal B$ is a stochastic kernel for $X_t | A_t, H_{t-1}$ in each round $t \ge 1$.\footnote{Allowing the environment to condition only on $H_{t-1}$ is without loss of generality: any additional private randomness, latent state, past potential outcomes or past shifts can be marginalised out via regular conditional probabilities, yielding an equivalent kernel for $X_t | A_t, H_{t-1}$.} Together, a sampling rule and a bandit instance induce a unique joint law over histories $(A_t, X_t)_{t \ge 1}$ that we will henceforth denote by $\pr$ and $\ex$, using the convenient shorthand $\ex_t[\cdot] \coloneq \ex\sbrcl{\cdot}{H_{t-1}}$.

To state our requirements about the bandit instance, we will assume that it is the marginal of a richer stochastic kernel, namely for $\bm X_t, B_t, X_t | A_t, H_{t-1}$, modelling the latent potential reward vector $\bm X_t \in \reals^K$ and the latent shift $B_t \in \reals$, neither of which is observed. We impose the standard assumptions of consistency $X_t = X_{t, A_t}$
 and independence $(\bm X_t, B_t) \indep A_t \mid H_{t-1}$, while explicitly allowing arbitrary dependence between $\bm X_t$ and $B_t$. We denote the shift-conditional mean reward of arm $a \in [K]$ by $\mu_{t,a} \coloneq \ex_t\sbrcl{X_{t,a}}{B_t}$. In this paper we additionally assume bounded means, sub-Gaussian noise and time-invariant gaps. More precisely we assume $\mathcal B$ admits a fixed range $U \ge 0$, a variance proxy $\sigma^2 \ge 0$ and a gap vector $\bm \Delta \in [0,2 U]^K$ with a uniquely attained minimum of $\min_{a \in [K]} \Delta_a = 0$ such that
\begin{assumption}\label{ass:main}
  For each $t \ge 1$ and each arm $a \in [K]$, almost surely
  \begin{enumerate}[label=(\alph*)]
  \item\label{it:a}
    the conditional mean $\mu_{t,a}$ is bounded\footnote{We show in Theorem~\ref{thm:lower bound} that without any constraints on the mean of the arm distributions, any $\delta$-correct algorithm would require an unbounded number of samples on average.} by $|\mu_{t,a}| \le U$,
  \item
    \label{it:b}
    the noise is conditionally $\sigma^2$-sub-Gaussian, i.e.
    $
    \ex_t\sbrcl{e^{\eta(X_{t,a} - \mu_{t,a})}}{B_t}
    \le
    e^{\frac{1}{2} \eta^2 \sigma^2}$
    for all $\eta \in \reals$,
  \item
    \label{it:c}
    the mean equals the shifted gap $\mu_{t,a} = B_t - \Delta_a$.
  \end{enumerate}
\end{assumption}
Note that by~\ref{it:c} the shift-conditional mean gap $\mu_{t,a} - \mu_{t,b}$ is the constant $(B_t - \Delta_a) - (B_t - \Delta_b) = \Delta_b-\Delta_a \eqcolon \Delta_{a,b}$. Also note that~\ref{it:b} does \emph{not} imply that $X_{t,a} - \ex_t[X_{t,a}]$ is $\sigma^2$-sub-Gaussian\footnote{
  Analogously, marginalising out $B_t$ raises variance by $\var_t[X_{t,a}] - \ex_t[\var_t\sbrcl{X_{t,a}}{B_t}] = \var_t\sbrl{\ex_t\sbrcl{X_{t,a}}{B_t}}$, i.e.  $\var_t[B_t]$.}, and~\ref{it:a} allows unbounded reward.
Our assumptions imply that there is a consistent and unique best arm, which we denote by%
\[
  a^* ~\coloneq~ \argmin_{a \in [K]} \Delta_a.
\]
We work in the fixed confidence setting, where a strategy for a learner consists of a sampling rule which induces action probabilities $w_{t,a} = \pr_t\setl{A_t = a}$, a stopping time $\tau \in \nats \cup \setl{\infty}$ with respect to the filtration $(H_t)_{t \ge 0}$, and an $H_\tau$-measurable recommendation rule $\tilde a \in [K]$.
We call a learner $\delta$-correct with confidence $\delta \in (0,1)$, if for any bandit instance satisfying Assumption~\ref{ass:main} we have
$
  \pr \setl{
    \text{$\tau < \infty$ and $\tilde a \neq a^*$}
  }
  ~\le~
  \delta.
$
Under this definition a learner that never stops is considered $\delta$-correct, which is intuitive when considering that never giving any recommendation will never produce a wrong recommendation.

We will look at efficient learners, and in particular develop $\delta$-correct learners with strong sample complexity guarantees. Specifically, for any bandit instance satisfying Assumption~\ref{ass:main} with some gap $\bm{\Delta}$, the sample complexity is bounded by
\[
  \ex[\tau]
  ~\le~
  \frac{\ln \frac{K-1}{\delta}}{D^*(\bm{\Delta})}
  +
  o\dell{\ln \frac{1}{\delta}}
\]
as $\delta \to 0$, where $D^*(\bm{\Delta})$ is a gap-dependent rate defined in \eqref{eqn:Dstar}. We will also derive a tight lower bound matching this upper bound in $\delta\to 0$ setting.
Although the rewards are conditionally stochastic, the latent shifts $B_t$ and the distribution of the $\sigma^2$-sub-Gaussian noise associated with each arm may be chosen in an adversarial, history-dependent manner. In this sense, our model lies between classical stochastic bandits and fully adversarial bandits. Randomisation plays a central role in this adversarial interaction. Here, unlike in the fully stochastic case, achieving \emph{any} upper bounds on sample complexity requires randomised arm selection, while \emph{nontrivial} lower bounds necessitate randomised shifts $B_t$.

\section{TAS and GLRT Fail}\label{sec:TASFails}
We argue that existing templates for algorithms, including Track-and-Stop \citep{garivier16optimalBAI}, do not address the Shifting Means setting. Under the mild assumption that the shifts can be larger than the gaps, the opponent may set $B_t = \Delta_{A_t}$ against deterministic $A_t$. Then all rewards $X_t$ are zero-mean regardless of the arm and reveal no information about the true gaps. Thus, any \emph{deterministic sampling} rule is powerless.

In addition, we argue that the classic \emph{GLRT stopping rule} also fails, even with shifted actually Gaussian rewards. There are two ways GLRT could be applied. First, we may simply ignore the shifts in its design. We argue that this is not $\delta$-correct as the randomised shifts can inflate the variance of the rewards. Not taking that increased variance into account leads to a higher error probability than $\delta$, meaning it is unsafe. The second way the GLRT can be applied is by including the possibility of shifts in its design. Now it becomes powerless, i.e.\ it never stops. The reason for $\tau=\infty$ is that this statistic overfits the latent shifts $B_t$ to the observed rewards $X_t$ each round; the statistic equals its starting value forever and hence never crosses its threshold. We analyse both these cases in detail in Appendix~\ref{sec:warmupimpossibilty}.

\section{Design and Sample Complexity Analysis of the \algname{} Strategy}\label{sec:alg}
In this section we will define and analyse our learning strategy. We will first set up the stopping and recommendation rule, and then consider the sampling rule.

\subsection{The \algname{} Stopping and Recommendation Rules}\label{sec:safety}
For each ordered pair of distinct arms $a \neq b$, we introduce a measure of evidence $z_{t,a,b}$ in round $t$ against $\Delta_{a,b} \le 0$, i.e.\ arm $b$ being better than arm $a$. It is parameterised by the sampling rule $\bm w_t$ and by the hyperparameters $\lambda_{t,a,b}$, $\alpha_{t,a,b}$, and $\beta_{t,a,b}$ which we will tune in Sections~\ref{sec:oracle} and~\ref{sec:power}. We define
\begin{equation}\label{eq:z}
  z_{t,a,b}
  ~\coloneq~
  \frac{\indicator_{t, a}}{w_{t, a}} \cdot \dell{
    \lambda_{t,a,b} X_t - \alpha_{t,a,b}
  }
  + \frac{\indicator_{t, b}}{w_{t, b}} \cdot \dell{
    -\lambda_{t,a,b} X_t - \beta_{t,a,b}
  }
\end{equation}
where $\indicator_{t,a} \coloneq \indicator[A_t = a]$ denotes the indicator for $A_t = a$. We further define the cumulative evidence after $t$ rounds by $Z_{t,a,b} \coloneq \sum_{s=1}^t z_{s,a,b}$, which is compared against a threshold depending on the confidence $\delta$ to yield the stopping rule. The full \algname{} stopping and recommendation rules are given by%
\begin{equation}\label{eq:stopping.and.recommendation}
  \tau
  ~\coloneq~
  \inf \setcl{
    t \in \nats
  }{
    \max_{a \in [K]} \min_{b \neq a} Z_{t,a,b} ~\ge~ \ln \frac{K-1}{\delta}
  }
  \qquad
  \text{and}
  \qquad
  \tilde a
  ~\coloneq~
  \argmax_{a \in [K]} \min_{b \neq a} Z_{\tau,a,b}
  ,
\end{equation}
with ties broken arbitrarily. We will only work with parameter values ensuring that $(e^{Z_{t,a,b}})_{t \ge 0}$ is a supermartingale whenever in fact $\Delta_{a,b} \le 0$, i.e.\ $b$ is better than $a$. To express that, define $\Phi$ for $w_a,w_b,\lambda \ge 0$ and $\alpha,\beta \in \reals$ by
\begin{equation}\label{eqn:def Phi}
  \Phi(w_a, w_b, \lambda, \alpha, \beta)
  ~\coloneq~
  1 - w_{a} - w_{b} + \max_{u \in \{-U, U\}} w_{a} e^{ \frac{\lambda u - \alpha}{w_{a}} + \frac{\lambda^2 \sigma^2}{2 w_{a}^2}}
  + w_{b} e^{\frac{-\lambda u - \beta}{w_{b}} + \frac{\lambda^2 \sigma^2}{2 w_{b}^2}}.
\end{equation}
We then have that $\Phi$ equals the worst-case exponential moment of $z_{t,a,b}$ over all bandit instances.
\begin{restatable}{relemma}{safetyconvex}\label{lem:safetyconvex}
  For all $t \ge 1$, arms $a \neq b$, sampling rule $\bm w_t$ and parameters $(\lambda_{t,a,b}$, $\alpha_{t,a,b}$, $\beta_{t,a,b})$,
  \[
    \Phi(w_{t,a}, w_{t,b}, \lambda_{t,a,b}, \alpha_{t,a,b}, \beta_{t,a,b})
    ~=~
    \max_{\mathcal B: Assumption~\ref{ass:main} \text{ holds}, \; \Delta_{a,b} \le 0}~
    \ex_t\sbrl{e^{z_{t,a,b}}}
    .
  \]
Moreover, $\Phi$ is jointly convex (in all five parameters).
\end{restatable}
With this, we can prove $\delta$-correctness, the proofs to both results can be found in Appendix~\ref{app:safety}.
\begin{restatable}{retheorem}{fullsafety}\label{lem:fullsafety}
  Any sampling rule $\bm w_t$, paired with stopping and recommendation rules~\eqref{eq:stopping.and.recommendation} employing predictable parameters $(\lambda_{t,a,b}, \alpha_{t,a,b}, \beta_{t,a,b})$ for each pair of arms $a \neq b$ and at each round $t$ such that $\Phi(w_{t,a}, w_{t,b}, \lambda_{t,a,b}, \alpha_{t,a,b}, \beta_{t,a,b}) \le 1$, yields a $\delta$-correct learner in any bandit satisfying Assumption~\ref{ass:main}.
\end{restatable}
The $\delta$-correctness proof is build on a round-by-round argument and thus does not rely on the assumption that gaps are fixed in time. The same proof would go through under the weaker assumption that the best arm is fixed (but with gaps that change arbitrarily in time), meaning that \algname\ is robust to that misspecified setting, although potentially suffering from a worse sample complexity.

\subsection{The Oracle Strategy for Known Gaps}\label{sec:oracle}
To have \eqref{eq:stopping.and.recommendation} stop and recommend the correct best arm $\tilde a = a^*$, we need $\min_{b \neq a^*} Z_{t,a^*,b}$ to cross its threshold. For each pair $a \neq b$, the cumulative evidence $Z_{t,a,b}$ has \emph{drift}, i.e.\ expected increment,
\[
  \ex_t\sbrl{z_{t,a,b}}
  ~=~
  \lambda_{t,a,b} \Delta_{a,b} - \alpha_{t,a,b} - \beta_{t,a,b}
  ,
\]
which depends on the chosen parameters and the true gap. We assume the gaps $\bm \Delta$ are known for this section and describe how to adapt to unknown gaps in the next Section~\ref{sec:power}. Our approach is to optimize the parameters and the sampling rule to maximise the minimal drift, since \algname{} only stops after the cumulative evidences for all arms have crossed the threshold. To ensure safety, i.e., to bound the probability of a wrong recommendation, the feasible parameters we optimize over have to observe the safety constraint as laid out in Theorem~\ref{lem:fullsafety}, meaning the parameters and the sampling rule are subject to a unit exponential moment constraint (sometimes called ``e-variable''\footnote{Compared to classic e-values our evidences are against sophisticated nulls. First of all, the shifts make our setting fork-convex (at least for fixed gaps $\bm \Delta$). Second of all, the fact that we actively control the arm $A_t$ makes that its randomness may be leveraged, which is what our importance weighted evidences actually exploit.}).

In other words, our goal is to solve
\begin{align}
  \label{eq:Dtab}
  D(w_a, w_b, \Delta)
  &~\coloneq~
  \max_{\substack{
      \lambda \ge 0,\ \alpha, \beta \in \reals
      \\
      \Phi(w_{a}, w_{b}, \lambda, \alpha, \beta)
      \le 1
    }}
  ~
  \lambda \Delta - \alpha - \beta
  \\
  \label{eqn:Dstar}
  D^*(\bm \Delta)
  &~\coloneq~
  \max_{\bm w \in \triangle_K} \min_{b \neq a^*}~ D(w_{a^*}, w_b, \Delta_b)
  \qquad
  \text{where}
  \qquad
  a^* = \argmin_{a \in [K]} \Delta_a~,
\end{align}
at which point the optimal parameters including the optimal sampling proportions can be read off. 

Both $D$ and $D^*$ are convex optimisation problems. To see why, observe that $D^*$ (which contains $D$) is a concave objective with convex constraints, and hence one can compute the solution numerically. We discuss an efficient special-purpose implementation in Appendix~\ref{sec:compute}, exploiting the particular structure of this problem. We prove useful analytic results about these problems in Appendix~\ref{app:drift}. For now, all that matters is that both~\eqref{eq:Dtab} and~\eqref{eqn:Dstar} are efficiently solvable.

For any gap vector $\bm \Delta$ with best arm $a^* = \argmin_{a \in [K]} \Delta_a$. We define the \emph{oracle learner} with sampling rule and parameters given by extracting the relevant maximizers
\begin{align*}
  \bm w_t &~\coloneq~ \argmax D^*(\bm \Delta)
  \\
  (\lambda_{t,a,b}, \alpha_{t,a,b}, \beta_{t,a,b}) &~\coloneq~ \argmax D(w_{t,a}, w_{t,b}, \Delta_{a,b}) \qquad \text{for all pairs $a \neq b$}~,
\end{align*}
which are in fact all independent of $t$. The oracle learner has maximised the minimal drift for all $b \neq a^*$, and in fact equalised all these drifts to $D^*(\bm \Delta)$. Intuitively, if the variance of each $z_{t,\astar,b}$ is not excessive, then after roughly $\logdeltaK / D^*(\bm \Delta)$ timesteps we should start to expect the stopping rule \eqref{eq:stopping.and.recommendation} to trigger. Based on this intuition, Lemma~\ref{lem:oracle} quantifies the stopping time of any sampling rule guaranteeing some nontrivial drift $D > 0$. Lemma~\ref{lem:drift ism} then supplies the required control on the (one-sided) sub-Gaussian variance proxy for the oracle sampling rule. Both proofs are in Appendix~\ref{app:power}.

\begin{restatable}{relemma}{oracle}\label{lem:oracle}
  Fix drift $D>0$ and variance proxy $\rho^2 \ge 0$ such that $\ex_t \sbrl{  e^{\xi \dell{ D - z_{t,\astar, b}} }} \leq e^{\frac{\xi^2 \rho^2}{2}}$ for every $\xi \ge 0$ and arm pair $b \neq \astar$. %
  Then the sample complexity of stopping rule~\eqref{eq:stopping.and.recommendation} is at most
  \begin{align*}
    \E[\tau]
    &~\le~
    1
    + \left(\sqrt{\frac{\logdeltaK}{D }+\frac{\rho ^2}{2 D ^2} (1 + \ln (K-1))}+\sqrt{\frac{\rho ^2}{2 D ^2} (1 + \ln (K-1))}\right)^2
    .
  \end{align*}
\end{restatable}
To interpret this bound, observe that $\lim_{\delta \to 0} \frac{\ex[\tau]}{\ln \frac{1}{\delta}} \le \frac{1}{D}$, so the asymptotic sample complexity is determined by the drift $D$. That is, the drift term dominates as the variance term is $\delta$ independent, which means we can ignore $\rho$. This motivates our choice \eqref{eqn:Dstar} of maximising the minimal drift across all arms within the safety constraint.

\subsection{The \algname{} Sampling Rule for Unknown Gaps}\label{sec:power}
The final step is to define the algorithm for unknown gaps. Our approach here is to estimate the gaps and plug them in. We use the unbiased importance-weighted estimator
\begin{align}\label{eq:unbiased}
  \hat\Delta_{t,a,b}
  ~\coloneq~
  \frac{1}{t}
  \sum_{s=1}^t
  \dell{
    \frac{\indicator_{s, a} }{w_{s, a}} - \frac{\indicator_{s, b}}{w_{s, b}}
  } X_s
  .
\end{align}
We extend the definition to $t=0$ by arbitrarily setting $\hat \Delta_{0,a,b} \coloneq 0$. At a high level, the approach of \algname{} is to tune in round $t$ as if the gaps $\hat{\Delta}_{t - 1, b} \coloneq \max_{a} \hat \Delta_{t - 1, a, b}$ were correct. This requires some care. While the true gap vector $\bm \Delta$ has a unique best arm (i.e.\ gap zero), and its entries are bounded in $[0,2 U]$, neither of which have to be true for the estimated gap vector $\hat{\bm \Delta}_{t-1}$. These are transient problems that vanish in probability as $t \to \infty$ (see e.g.\ Lemma~\ref{lem:deltaconc}) with sufficient exploration. If entries of $\hat{\bm \Delta}_{t-1}$ are above $2U$, we project them down to $2U$ componentwise, written $\hat{\bm \Delta}_{t-1} \wedge 2U$; as $\bm \Delta \in [0, 2U]^K$ this only moves the estimate closer to the truth, and it keeps $\lambda, \alpha$ and $\beta$ bounded. If there are several zeroes, then $D(w_a, w_b, 0) = 0$ and any weights $\bm w$ are optimal, in which case we choose to set $\bm w$ uniform. We set
\begin{subequations}
\begin{align}
  \label{eq:wtilde}
  \tilde{\bm w}_t
  &~\coloneq~
  \begin{cases}
    \argmax D^*(\hat{\bm \Delta}_{t-1} \wedge 2U)
    & \text{if $\hat{\bm \Delta}_{t-1}$ has a unique zero entry,}
    \\
    \bm 1_K / K
    & \text{otherwise,}
  \end{cases}
  \\
  \label{eq:w}
  \bm w_t
  &~\coloneq~
  \dell{1-\frac{1}{\sqrt{t}}} \tilde{\bm w}_t + \frac{\bm 1_K}{K \sqrt{t}}
  \\
  \label{eq:params}
  (\lambda_{t,a,b}, \alpha_{t,a,b}, \beta_{t,a,b})
  &~\coloneq~
  \argmax D(w_{t,a}, w_{t,b}, \hat \Delta_{t-1,a,b} \wedge 2U)
  \qquad \text{for all pairs $a \neq b$~.}
\end{align}
\end{subequations}
Here, $\tilde{\bm w}_t$ maximises the minimal drift as if $\hat{\bm \Delta}_{t-1} \wedge 2U$ were the real gaps, and \eqref{eq:w} mixes in exploration to obtain the actual sampling probabilities $\bm w_t$. With that the algorithm, which we call \algname, is completely defined. See Algorithm~\ref{alg:ism} for a summary. \algname{} is $\delta$-correct by Theorem~\ref{lem:fullsafety}. The next result shows that $D^*(\bm \Delta)$ quantifies its sample complexity.

\begin{algorithm}[t]
  \caption{Importance Weights for Shifting Means (\algname)}\label{alg:ism}
  \begin{algorithmic}[1]
  \REQUIRE Confidence $\delta \in (0,1)$, mean range $U > 0$, sub-Gaussianity constant $\sigma^2 > 0$
    \FOR{$t = 1, \dots$}
    \STATE Compute sampling proportions $\bm w_t$ in~\eqref{eq:w} using~\eqref{eq:wtilde} and~\eqref{eqn:Dstar}
    \STATE Sample arm $A_t \sim \bm w_t$ and observe reward $X_t$
    \STATE \label{line:params}
    For all pairs $a \neq b$, compute $(\lambda_{t,a,b}, \alpha_{t,a,b}, \beta_{t,a,b})$ in \eqref{eq:params} using~\eqref{eq:Dtab}
    \STATE \label{line:ev.n.gap}
    For all pairs $a \neq b$, update evidences $Z_{t,a,b}$ in~\eqref{eq:z} and gap estimates $\hat \Delta_{t,a,b}$ in~\eqref{eq:unbiased}
    \STATE \textbf{if} stopping condition \eqref{eq:stopping.and.recommendation} reached (i.e.\ $\tau = t$) \textbf{then} Stop and recommend arm $\tilde a$ in~\eqref{eq:stopping.and.recommendation}
    \ENDFOR
  \end{algorithmic}
\end{algorithm}

\begin{restatable}{retheorem}{samplecplx} \label{thm:sample complexity}
  Fix a bandit instance satisfying Assumption~\ref{ass:main} with gap $\bm \Delta$. \algname{} (Algorithm~\ref{alg:ism}) ensures asymptotic sample complexity
  \[
    \lim_{\delta \to 0} \frac{\ex[\tau]}{\logdelta}
    ~\le~
    \frac{1}{D^*(\bm \Delta)}.
  \]
\end{restatable}

The proof needs to overcome two complications compared to that of Lemma~\ref{lem:oracle}. First, the gaps are now learned, and as such the drift is not uniformly bounded below over time. We argue, by concentration, that the true gap is learned quickly enough. In addition, the sub-Gaussian variance proxy used in Lemma~\ref{lem:oracle} can get arbitrarily large when the sampling weights $w_{t,a}$ get small. Forced exploration provides a lower bound of order $1/\sqrt{t}$, yet this would be vacuous when combined with Hoeffding's range bound, which scales as $\frac{1}{w_{t,a}^2}$. We overcome this by relying on Bernstein instead, as the actual variance scales by only $\frac{1}{w_{t,a}}$. We refer the reader to Appendix~\ref{sec:scproof} for a complete proof, and conclude with showing that \algname\ is efficiently implementable, proof in Appendix~\ref{sec:compute}.
\begin{lemma}
  \algname{} requires $O(K^2)$ storage, and $O(K)$ work per round.
\end{lemma}

\begin{proof}
  First of all, instead of maintaining $\hat \Delta_{t,a,b}$ for all pairs $a \neq b$, we can instead track the terms in \eqref{eq:unbiased} for arm $a$ and arm $b$ separately. This requires maintaining $K$ quantities instead of $K(K-1)/2$ for the gap estimates, though the evidence need to be stored separately leading to a $O(K^2)$ storage requirement. Lines~\ref{line:params} and~\ref{line:ev.n.gap} mention all $K(K-1)$ ordered distinct pairs. But due to the indicators in the evidence \eqref{eq:z} and gap estimates \eqref{eq:unbiased}, only the $2(K-1)$ pairs containing $A_t$ need updating in line~\ref{line:ev.n.gap}, and only their parameters need to be computed in line~\ref{line:params}. The weights and parameters can be computed in $O(K)$ work as described in Appendix~\ref{sec:compute}.
\end{proof}

\subsection{Interpretable Bounds}
While Theorem~\ref{thm:sample complexity} and $D^*(\bm \Delta)$ provide an accurate quantification of the expected stopping time, it is not apparent what the scaling of critical parameters $\sigma^2$, $U$ and the gaps $\bm \Delta$ actually is. Therefore, we present interpretable bounds on the sample complexity for our algorithm in this section by lower and upper bounding $D^*(\bm \Delta)^{-1}$.

\begin{restatable}{retheorem}{lemESbound}\label{thm:ESbound}
  Define the second best action by $b^* = \argmin_{b \neq\astar} \Delta_b$. Then $D^*(\bm \Delta)$ as defined in \eqref{eqn:Dstar} is sandwiched by (where the lower bound $(\star)$ holds whenever $\max_a \Delta_a \leq \frac{1}{2}\gamma U$, with $\gamma \coloneq 3(\coth(1) - 1) \approx 0.94$)
  \[
    2 (\sigma^{2} + \gamma U^2) \dell{
      \frac{1}{\Delta_{{b^*}}^2} + \sum_{a \in [K]\setminus \astar}\frac{1}{\Delta_a^2}
    }
    ~\stackrel{(\star)}{\leq}~
    \frac{1}{D^*(\bm \Delta)}
    ~\leq~
    4 (\sigma^{2} + U^2)\dell{
      \frac{1}{\Delta_{{b^*}}^2} + \sum_{a \in [K]\setminus \astar}\frac{1}{\Delta_a^2}
    }
    .
  \]
\end{restatable}
The proof of Theorem~\ref{thm:ESbound} is based on technical calculation to derive a term that is similar to the result of \citet{garivier16optimalBAI} and indeed, we obtain essentially the same upper bound on $D^*(\bm \Delta)^{-1}$ as they do, except for the presence of $U$. This is perhaps not surprising as illustrated by our lower bounds presented in the next Section~\ref{sec:lower bounds}, that also feature the factor $U$ on the sample complexity of any $\delta$-correct algorithm. The two sides of Theorem~\ref{thm:ESbound} differ by a factor of $2/\gamma \approx 2.1$. Its lower bound on $1/D^*(\bm \Delta)$ only holds if $2 \max_a \Delta_a \le \gamma U$, i.e., if the adversary has comparatively more power for shifting; if $\max_a \Delta_a = 2U$ the range of the shifts is constant over time. \algname\ recovers the optimal bound in the BAI setting and we leave showing an exact graceful degradation of performance to future work.

\section{Lower bounds} \label{sec:lower bounds}
In this section, we prove a lower bound on the expected number of samples any $\delta$-correct algorithm for the shifting means setting introduced in Section~\ref{sec:setup} will need to request. We define our lower-bound bandit instances $\mathcal B$ as follows: Given a gap vector $\bm \Delta$ and a shift distribution\footnote{For deterministic $B_t$, one has to assume that the learner has full knowledge of $B_t$, just like how the learner has to assume that an adversary has full knowledge of the learning algorithm. In that case, the learner can simply subtract $X_t - B_t$ and they will recover the standard BAI setting, where the lower bounds feature no dependency on $U$.} for $B_t$, we sample the potential rewards from a normal distribution $X_{t, a}|B_t \sim \Ncal(B_t - \Delta_a, \sigma^2)$ with fixed (and known) variance $\sigma^2$. To ensure that $\mathcal B$ fulfills Assumption~\ref{ass:main}, the shift distribution has to be bounded, that is $B_t \in [\max_a \Delta_a - U, U]$ with probability 1.
The main challenge of the lower bound is to choose a shift distribution that makes the problem as hard as possible. As in the lower bound of \citet{garivier16optimalBAI}, the hardness of a bandit instance is measured by how well a single reward distinguishes two instances with a different best arm. That is, informally, the expected stopping time $\E[\tau]$ is determined by the fraction of pulls $w_a$ that any algorithm allocates to arm $a$ and a one-round KL divergence
\begin{align}\label{eqn:lower informal}
  \E[\tau] ~\gtrsim~ \logdelta \frac{1}{ \sum_{a = 1}^K w_a\, \KL{\epsilon + B}{\Delta_a - \tilde \Delta_a + \epsilon + B} }~,
\end{align}
for any $\tilde{\bm \Delta}$ with a different best arm than $\bm \Delta$ (full details in Theorem~\ref{thm:lower bound}/Appendix~\ref{app:lower bounds}). When picking $B = 0$, we exactly recover the standard lower bound for Gaussian bandits \citep{garivier16optimalBAI}. The goal now becomes to leverage the shift $B$ to make this KL divergence as small as possible while respecting the boundedness required by Assumption~\ref{ass:main}.

For small $\triangle \coloneq \Delta_a - \tilde \Delta_a$, the KL divergence between nearby members of a smooth location family is governed by its Fisher information. Writing $I(X) \coloneq \int q'(x)^2 / q(x) \, d x$ for the Fisher information of a random variable $X$ with density $q$, we have
\begin{align*}
  \KL{\epsilon + B}{\triangle + \epsilon + B} ~\approx~ \frac{\triangle^2}{2}\, I(\epsilon + B) ~<~ \frac{1}{2} \frac{\triangle^2}{ \sigma^2 + 1 / I(B) }~,
\end{align*}
where we plugged in $I(\epsilon) = 1/\sigma^2$ for the normal distribution, the critical term remaining being the Fisher information of the shifts $I(B)$; the details can be found in the proof of Lemma~\ref{lem:kl inflate} in Appendix~\ref{app:lower bounds}.

Thus for our shifting distribution, we seek to minimize the Fisher information. The minimizer on a closed interval $[-R, R]$ is given by a \emph{raised cosine distribution}
\begin{equation}\label{eqn:hardshiftdist}
    p(B') ~=~
    \begin{cases}
     \frac{1}{R} \cos^2\!\Big( \frac{\pi B' }{2 R} \Big) & \text{if } |B'| \leq R \\
     0 & \text{otherwise}~,
    \end{cases}
\end{equation}
which achieves $I(B') = \pi^2 / R^2$ for $B' \sim p$. Overall, that allows us to bound the KL as follows:
\begin{restatable}{relemma}{lemklinflate}\label{lem:kl inflate}
  Fix $\sigma^2 > 0$ and $R > 0$. Let $\epsilon \sim \Ncal(0, \sigma^2)$ and $B' \sim p$ as in Equation~\eqref{eqn:hardshiftdist} be independent. There exists $d > 0$, depending only on $\sigma^2$ and $R$, such that for all $|\triangle| \le d$,
  \begin{align*}
    \KL{\epsilon + B'}{\triangle + \epsilon + B'} ~\le~ \frac{\triangle^2}{2 (\sigma^2 + R^2 / \pi^2)}~.
  \end{align*}
\end{restatable}
The proof can be found in Appendix~\ref{app:lower bounds}. Formalizing Equation~\ref{eqn:lower informal} and combining it with the last lemma yields the main result of this section:

\begin{restatable}{retheorem}{thmlowerbound} \label{thm:lower bound}
  Fix $U > 0$, $\sigma^2 > 0$, and fix any algorithm that is $\delta$-correct for all bandit instances for the given $U$ and $\sigma^2$. There exists a set of gaps $\bm \Delta$ and a shift distribution, such that the following lower bound holds:
    \begin{align*}
      \liminf_{\delta \to 0} \frac{\E[\tau]}{\logdelta} ~\ge~ 2 \l( \sigma^2 + \frac{(U - \max_a \Delta_a / 2)^2}{\pi^2} \r) \l( \frac{1}{\Delta_{{b^*}}^2} + \sum_{a \in [K]\setminus \astar}\frac{1}{\Delta_a^2} \r)~.
    \end{align*}
\end{restatable}
Comparing with the upper bound $4(\sigma^2 + U^2)$ of Theorem~\ref{thm:ESbound} shows that, for gaps that are small relative to $U$, \algname\ is optimal up to a factor of at most $2\pi^2 \approx 20$ and approaches $2$ when $\sigma^2$ dominates $U^2$.

\section{Numerical Experiments}\label{sec:numerics}
We consider a $K=4$ armed bandit instance\footnote{Full source code: \url{https://github.com/LukasZierahn/Shifting-Means-Experiments}} with Bernoulli %
rewards and fixed mean gap vector
$\vec \Delta ~=~ \begin{bmatrix}
    0.0&  0.05&  0.07&  0.1
  \end{bmatrix}$,
which coincides with the gaps of the stationary bandit $\vec \mu_1 = [0.5, 0.45, 0.43, 0.4]$ used in the experiments of \citet{garivier16optimalBAI}. We work with adversarially shifted means $\mu_{t,a} = B_t - \Delta_a$, where we draw the per-round common shifts $B_t$ i.i.d.\ from the uniform distribution over the range $[0.1, 1]$. Accordingly, we set $U=1$
and $\sigma^2=1/4$. We solve \eqref{eqn:Dstar} to obtain optimal drift and weights
\[
  D^* ~=~ \frac{1}{4965.74}
  \qquad
  \text{and}
  \qquad
  \vec w^* ~=~ \begin{bmatrix}
    0.4168&  0.3896&  0.1363&  0.0572
  \end{bmatrix}
  .
\]
Our optimal weights are to first order the same as those reported by \citet{garivier16optimalBAI} meaning the shifts only have a small influence on the sampling probabilities in this case. We also display the optimal values for the evidence parameters in Table~\ref{tab:values}. As the gap estimates only change slowly, we speed up the computation by recomputing the $\bm w_t$ only every 100 timesteps. Table~\ref{tab:stopping times} and Figure~\ref{fig:ism_e-1_delta} in Appendix~\ref{app:numerical} show the expected and realized stopping times of \algname\ over 1000 iterations (All experiments were run on a personal desktop (CPU: Intel i7-10700K) with a total compute time of less than 24h.). We can see how the additional fraction of samples between empirical and predicted stopping time (Theorem~\ref{thm:sample complexity}) vanishes as $\delta$ shrinks.

Since we are using the exact same setup as \citet{garivier16optimalBAI} (modulo shifts), we can compare numbers directly. The mean stopping time with $\delta = \frac{1}{10}$ for TAS is 3968 or 4052 (averaged over 3000 iterations) depending on the tracking strategy, while the empirical stopping time of \algname\ is between [28 085, 29 470] with 95\% confidence.

Figure~\ref{fig:ism_e-1_delta} also shows that the algorithm is at least empirically robust to misspecification of $U$ as all iterations over different $U$ and $\delta$ values find the correct arm, except for 4 out of the 1000 runs with $U = \frac{1}{4}$.

\begin{table}
  \centering
  \begin{tabular}{lrrrrr}
    $\delta$ & $\frac{1}{D^*} \ln \frac{1}{\delta}$ & $\frac{1}{D^*} \ln \frac{K-1}{\delta}$ & Lemma~\ref{lem:oracle} & Empirical & Empirical Overhead \\
    \hline
    $10^{-1}$ &  11\,434 &  16\,889 &  69\,702 & [28\,085, 29\,470] & [145.62\%, 157.74\%]\\
    $10^{-10}$ & 114\,340 & 119\,796 & 211\,870 & [140\,868, 143\,920] & [23.20\%, 25.87\%]\\
    $10^{-100}$ & 1\,143\,403 & 1\,148\,858 & 1\,384\,182 & [1\,199\,222, 1\,207\,930] & [4.88\%, 5.64\%]
  \end{tabular}
  \caption{Theoretical predictions and practical measurement of sample complexity for the bandit as laid out in Section~\ref{sec:numerics} at different $\delta$. For Lemma~\ref{lem:oracle} we took as the variance proxy $\rho^2$ the largest worst-case variance in Table~\ref{tab:values}. The empirical value is given as a 95\% confidence interval of 1000 samples. The ratio is the real overhead not predicted by theory, that is, it is the relative percentage difference between the first and second last entry of the table.}\label{tab:stopping times}
\end{table}

\section{Discussion \& Conclusion}\label{sec:conclusion}
We introduced the Shifting Means setting and demonstrated how previous approaches using the GLRT fail in this setting. Then we developed \algname\ and showed that it is $\delta$-correct and enjoys a sample complexity guarantee of $4 (\sigma^2 + U^2) \big( \Delta_{b^*}^{-2} + \sum_{a \in [K] \setminus \astar} \Delta_a^{-2} \big) \logdelta$ (Theorem~\ref{thm:ESbound}). We demonstrated the practicality of \algname, and empirically demonstrate its robustness in misspecification of $U$. Furthermore, we developed a novel construction for our lower bound, which matches the upper bound up to a constant factor (Theorem~\ref{thm:lower bound}) of at most $2\pi^2 \approx 20$ for gaps that are small relative to $U$.

Shifting Means is a strict generalization of the traditional BAI setting, though our safety argument holds for the even larger class of problems where only the identity of the best arm is fixed each timestep. That is, we are safe when the mean rewards of all arms change arbitrarily over a bounded range, as long as the best arm always has the highest expected payoff. However, approaching that setting from the power side is highly non-trivial. Our current approach is based around plugging in the estimated gaps, which are then not fixed any more, will change adversarially over time, and can even vanish. Showing, for example, convergence guarantees based on the time averaged gap remains a major challenge. We leave that very interesting setting open for future work.

\begin{ack}
  DvdH was supported by Netherlands Organization for Scientific Research (NWO), grant number VI.Veni.242.143 (\url{https://doi.org/10.61686/SCGQQ55380}).
\end{ack}

\bibliographystyle{plainnat}
\bibliography{bib}

\appendix

\section{TAS and GLRT fail in the Shifting Means setting}\label{sec:warmupimpossibilty}
This appendix contains two parts, first we sketch some informal arguments to illustrate why the celebrated Track-and-Stop (TAS) \citep{garivier16optimalBAI} algorithm may no longer be $\delta$-correct in the Shifting Means setting and second, we show that a na\"{\i}ve adaptation of TAS has no power.

\subsection{TAS Fails}
First we will argue that TAS, when tuned for the classical $K$-armed bandit with Gaussian arms of unknown means and unit variance (henceforth referred to as the 1-Gaussian bandit), may incur more than a $\delta$ fraction of errors when applied to a Shifting Means 1-Gaussian bandit. Our reasoning proceeds in two steps, which we outline below.\\

\noindent{\bf Step 1. (Informal)} Consider a $1$-Gaussian bandit with 2 arms. Since the adversary can randomize while selecting the shifts, if it randomizes according to a $0$-mean Gaussian with unit variance, then the samples that the algorithm observes from any arm, say $a$, are draws from a Gaussian with same mean, but inflated variance ($2$ in this case). Thus, it suffices to show that TAS tuned for 1-Gaussian bandits, when run on $\sigma^2$-Gaussian bandits is not $\delta$-correct.

Note that randomizing according to a Gaussian distribution is not a valid scheme for the adversary since the resulting mean may no longer respect the bounds of $[0, U]$. But for large $U$, this is a good model to get intuitions. \\

\noindent{\bf Step 2. (Informal)} Suppose we observe samples $X_1, X_2 \dots $ that are distributed iid according to $N(m,1)$. We want to test the following hypothesis:
\[ \mathcal H_0: m=0 \qquad \text{vs} \qquad \mathcal H_1: m = 1, \numberthis\label{eq:hypothesis} \]
sequentially using a $\delta$-correct, power-1 stopping rule $\tau_\delta$. Formally, this means that the stopping rule should satisfy the following: if samples come from $\mathcal H_0$, then $\Pr_0[\tau_\delta < \infty] \le \delta$. This corresponds to $\delta$-correctness requirement. Next, when the samples come from $\mathcal H_1$, then $\Pr_1[\tau_\delta < \infty] = 1$. This corresponds to power-1 part of the requirement.

For this problem, the GLRT rule (tuned to 1-Gaussian distributions) reduces to the following
\[ \tau_\delta = \min\left\{n: e^{nS_n - \frac{n\theta^2}{2}} \ge \frac{1}{\delta} \right\}, \quad \text{ where }\quad S_n:= \sum\limits_{i=1}^n X_i. \]
$\delta$-correctness easily follows from an application of Ville's inequality \citep{ville1939etude}.

We now show that the above GLRT stopping rule is not $\delta$-correct for testing problem~\eqref{eq:hypothesis} with data $X_i$ i.i.d.\ according to $N(m,\sigma^2)$, for $\sigma > 1$. To this end, consider the following. Note that the following probabilities are evaluated under $X_i \sim N(0, \sigma^2)$.

\begin{align*}
    \Pr\left[ \tau_\delta < \infty \right]&=\Pr\left[\exists n \in \nats: \l\{\theta S_n - \frac{\theta^2 n}{2} \ge \log \frac{1}{\delta} \r\} \right]\\
    &=\Pr\left[\exists n \in \nats: \l\{ S_n \ge  \frac{\theta n}{2} + \frac{\log \frac{1}{\delta}}{\theta} \r\} \right]\\
    &\approx\Pr\left[\exists t \in \reals_+: \l\{ S_t \ge  \frac{\theta t}{2} + \frac{\log \frac{1}{\delta}}{\theta} \r\} \right]\\
    &= \delta^\frac{1}{\sigma^2} > \delta. \tag{$\sigma > 1$}
\end{align*}

\subsection{Na\"{\i}ve TAS Extension}
Since the TAS stopping rule may fail in the Shifting Means setting, a natural question is whether a straightforward adaptation of the Generalized Likelihood Ratio Test (GLRT) could serve as a valid stopping rule. As the name suggests, the GLRT is based on the likelihood ratio of the null hypothesis that action $a$ is not better than action $b$ against the alternative hypothesis that action $b$ is better than $a$. In our setting, we can write the GLRT based test statistic at time $t$ as
\begin{align}\label{eqn:glrtz}
  \glrtz_{a, b}(t) &\coloneq \log \l( \frac{\sup_{\substack{\mubar' \in \reals^{T \times \Acal} \\ \mu'_{t, a} > \mu'_{t, b}}} \prod_{s = 1}^t \mathbb{P}(X_s | \mu'_{s, A_s})}{\max_{\substack{\mubar' \in \reals^{T \times \Acal} \\ \mu'_{t, a} \leq \mu'_{t, b}}} \prod_{s = 1}^t \mathbb{P}(X_s | \mu'_{s, A_s})} \r)~.
\end{align}

Lemma~\ref{lem:glrt} shows that this is not the case: in the Shifting Means setting, such a GLRT-based rule will, with high probability, fail to stop altogether. On a high level, the GLRT fails at the timestep $t$ as there are many more parameters ($Kt$ effectively independent parameters) to estimate while only having access to at most $t$ data points, which leads to catastrophic overfitting.

\begin{restatable}{relemma}{lemmaglrt} \label{lem:glrt}
  The GLRT fails to stop for shifting means as the test statistic is uniformly $0$ at all times, irrespective of the observed data.
\end{restatable}

\begin{proof}
  Let $A_s$ be any action sampled at time-step $s$ and $X_s \in  \reals$ be any reward value that has non-zero probability under arm distribution of arm $A_s$ at time $s$, for all $s \in [t]$. 

  We start by investigating the denominator of the GLRT, as it is defined in~\eqref{eqn:glrtz}, given by
  \begin{align*}
    \max_{\substack{\mubar' \in \reals^{T \times \Acal} \\ \mu'_{t, a} \leq \mu'_{t, b}}} \prod_{s = 1}^t \mathbb{P}(X_s | \mu'_{s, A_s}) = \prod_{s = 1}^t \max_{\substack{\mubar' \in \reals^{\Acal} \\ \mu'_{a} \leq \mu'_{b}}} \mathbb{P}(X_s | \mu'_{A_s}) = \prod_{s = 1}^t \max_{\mu' \in \reals} \mathbb{P}(X_s | \mu')~,
  \end{align*}
  where the first equality follows as $\mu_{t, c}$ only affects round $t$ and the second equality follows by setting $\mu'_a = \mu'_b$ to fulfill the constraint and then deleting all $\mu'_c$ for $c \neq A_t$, as they have no influence on the optimization problem. The numerator follows the same template
    \begin{align*}
    \sup_{\substack{\mubar' \in \reals^{T \times \Acal} \\ \mu'_{t, a} > \mu'_{t, b}}} \prod_{s = 1}^t \mathbb{P}(X_s | \mu'_{s, A_s}) = \prod_{s = 1}^t \sup_{\substack{\mubar' \in \reals^{\Acal} \\ \mu'_{a} > \mu'_{b}}} \mathbb{P}(X_s | \mu'_{A_s}) = \prod_{s = 1}^t \max_{\mu' \in \reals} \mathbb{P}(X_s | \mu')~,
  \end{align*}
  where the same solution of setting $\mu'_a = \mu'_b$ is on the edge of the feasible set and thus valid for the supremum.

  Putting both of these equations together gives that
  \begin{align*}
      \glrtz_{a, b}(t) &\coloneq \log \l( \frac{\sup_{\substack{\mubar' \in \reals^{T \times \Acal} \\ \mu'_{t, a} > \mu'_{t, b}}} \prod_{s = 1}^t \mathbb{P}(X_s | \mu'_{s, A_s})}{\max_{\substack{\mubar' \in \reals^{T \times \Acal} \\ \mu'_{t, a} \leq \mu'_{t, b}}} \prod_{s = 1}^t \mathbb{P}(X_s | \mu'_{s, A_s})} \r)
      = \log \l( \frac{\prod_{s = 1}^t \max_{\mu' \in \reals} \mathbb{P}(X_s | \mu')}{\prod_{s = 1}^t \max_{\mu' \in \reals} \mathbb{P}(X_s | \mu')} \r)
      = 0~.
  \end{align*}
\end{proof}

\section{Upper Bound - Stopping and Recommendation Rules - Details}\label{app:safety}

\safetyconvex*

\begin{proof}
  We first show that $\Phi$ is jointly convex in $w_a, w_b, \lambda, \alpha$, and $\beta$. We define the function $f$ with constants $c_1, c_2$ where $c_2 > 0$ as
  \begin{align*}
    f(x, y) = \exp \l( x c_1 - y + x^2 c_2 \r)~,
  \end{align*}
  and observe that this function is jointly convex in $x$ and $y$. The perspective transform of a function is defined as $g(x, y, w) \to w f(x / w, y / w)$ and is jointly convex in $x, y$ and $w \ge 0$ if the original function $g$ is jointly convex in $x, y$. Noting that $f$ is jointly convex in its parameters and then applying the perspective transform yields a function jointly convex in $x, y$, and $w$
  \begin{align*}
    g(x, y, w) \coloneq w \exp \l( \frac{x c_1 - y}{w} + \frac{x^2}{w^2} c_2 \r)~.
  \end{align*}

  Finally, we are in a position to show that $\Phi$ is convex after recalling that the sum and max of convex functions is also convex and recognizing $g$ twice
  \begin{align*}
    \Phi(w_a, w_b, \lambda, \alpha, \beta) &= (1 - w_{a} - w_{t, b}) + \max_{u \in \{-U, U\}}
    \underbrace{w_{a} \exp \l( \frac{\lambda u - \alpha}{w_{a}} + \frac{\lambda^2}{w_{a}^2} \frac{\sigma^2}{2} \r)}
    _{g(\lambda, \alpha, w_{a})}\\
    &\quad + 
    \underbrace{w_{b} \exp \l(\frac{-\lambda u - \beta}{w_{b}} + \frac{\lambda^2}{w_{b}^2} \frac{\sigma^2}{2} \r)}
    _{g(\lambda, \beta, w_{b})}~,
  \end{align*}
  with $c_1 = u$ (or $c_1 = - u$) and $c_2 = \sigma^2 / 2 > 0$.

  For the first claim, we start by considering the expectation of the RHS for any given $\mathcal B$ and apply a tower rule
  \begin{align*}
    \max_{\mathcal B: Assumption~\ref{ass:main} \text{ holds}, \; \Delta_{a,b} \le 0}~
    \ex_t\sbrl{e^{z_{t,a,b}}}
    = \max_{\mathcal B: Assumption~\ref{ass:main} \text{ holds}, \; \Delta_{a,b} \le 0}~
    \ex_t \Big[ \ex_t \l[ e^{z_{t,a,b}} \given* B_t \r] \Big]~.
  \end{align*}
  
  We isolate the innermost expectation and write out the definition of the evidence (Equation~\ref{eq:z}), resolve the expectation over $A_t$, and follow it up with applying Assumption~\ref{ass:main}\ref{it:b}
  \begin{align*}
    &\E_t \l[ e^{z_{t,a,b}} \given* B_t \r]
    \\
    &= \E_t \l[ e^{ \frac{\indicator_{t, a}}{w_{t, a}} \cdot \dell{\lambda_{t, a, b} X_t - \alpha_{t, a, b}}
    + \frac{\indicator_{t, b}}{w_{t, b}} \cdot \dell{
    -\lambda_{t, a, b} X_t - \beta_{t, a, b}} } \given* B_t \r]\\
    &= (1 - w_{t, a} - w_{t, b}) + w_{t, a} \E_t \l[ e^{ \frac{\lambda_{t, a, b} X_t - \alpha_{t, a, b}}{w_{t, a}} } \given* B_t, A_t = a \r]\\ 
    &\qquad \qquad + w_{t, b} \E_t \l[ e^{ \frac{- \lambda_{t, a, b} X_t - \beta_{t, a, b}}{w_{t, b}} } \given* B_t, A_t=b \r]\\
    &\leq (1 - w_b - w_a) + w_{t, a} e^{ \frac{\lambda_{t, a, b} \mu_{t, a} -\alpha_{t, a, b}}{w_{t, a}} + \frac{\lambda_{t, a, b}^2}{w_{t, a}^2} \frac{\sigma^2}{2} } + w_{t, b} e^{ \frac{ -\lambda_{t, a, b} \mu_{t, b} - \beta_{t, a, b}}{w_{t, b}} + \frac{\lambda_{t, a, b}^2}{w_{t, b}^2} \frac{\sigma^2}{2} }~.
  \end{align*}
  The above bound holds for all $\bm \mu_t$ where $\Delta_{b} \leq \Delta_{a}$, which we now worst case over. First we recognize that the max of an expectation is achieved for a non-random value of $\bm \mu$, then that the value is increasing in $u_a$ but decreasing in $u_b$ and since $u_a \leq u_b$, we can equivalently optimize over the point $u$ where they meet
  \begin{align*}
    &\max_{\mathcal B: Assumption~\ref{ass:main} \text{ holds}, \; \Delta_{a,b} \le 0} \E_t \l[ w_{t, a} e^{\frac{\lambda_{t, a, b} \mu_{t, a} -\alpha_{t, a, b}}{w_{t, a}} + \frac{\lambda_{t, a, b}^2}{w_{t, a}^2} \frac{\sigma^2}{2} } + w_{t, b} e^{ \frac{ -\lambda_{t, a, b}\mu_{t, b} - \beta_{t, a, b}}{w_{t, b}} + \frac{\lambda_{t, a, b}^2}{w_{t, b}^2} \frac{\sigma^2}{2} } \r]\\
    &= \max_{u_a \leq u_b, \; |u_a|, |u_b| \leq U} w_{t, a} e^{\frac{\lambda_{t, a, b} u_a -\alpha_{t, a, b}}{w_{t, a}} + \frac{\lambda_{t, a, b}^2}{w_{t, a}^2} \frac{\sigma^2}{2} } + w_{t, b} e^{ \frac{ -\lambda_{t, a, b}u_b - \beta_{t, a, b}}{w_{t, b}} + \frac{\lambda_{t, a, b}^2}{w_{t, b}^2} \frac{\sigma^2}{2} } \\
    = &\max_{u \in [-U, U]} w_{t, a} e^{\frac{\lambda_{t, a, b} u -\alpha_{t, a, b}}{w_{t, a}} + \frac{\lambda_{t, a, b}^2}{w_{t, a}^2} \frac{\sigma^2}{2} } + w_{t, b} e^{ \frac{-\lambda_{t, a, b} u - \beta_{t, a, b}}{w_{t, b}} + \frac{\lambda_{t, a, b}^2}{w_{t, b}^2} \frac{\sigma^2}{2} }~,
  \end{align*}
  We obtain the final result as the maximizer of a convex function is on the boundary of the range and we can recognize $g$ twice more to see the value being convex in $u$.
\end{proof}

\fullsafety*

\begin{proof}
  This proof consists of two parts, first we construct a test for two arms and then we do an appropriate union bound. So, fix $a, b \in \Acal$, with $a \neq b$ and $\Delta_{a, b} \leq 0$. Since arm $a$ is better than arm $b$ we want to show it is unlikely for our evidence to grow large at any point in time. Specifically, we require that for certainty $\delta'$, it holds that
  \begin{align*}
    \Prob \l( \exists t: Z_{t, a, b} \geq \logdeltadash \r) = \Prob \l( \exists t: \exp \l( Z_{t, a, b} \r) \geq \frac{1}{\delta'} \r).
  \end{align*}
  The main tool in this proof is Ville's inequality, which states that
  \begin{align*}
    \Prob \l( \exists t: \Mcal_t \geq \frac{1}{\delta'} \r) \leq \delta'
  \end{align*}
  for any super-martingale $\Mcal_t$ with $\Mcal_0 = 1$, in our case $\Mcal_t \coloneq \exp(Z_{t, a, b})$ with $Z_{t, a, b} \coloneq \sums z_{s, a, b}$. Since $\exp(Z_{0, a, b}) = \exp(0) = 1$, the only thing that is left to do for us is to check that $\exp(Z_{t, a, b})$ is a super-martingale, for which most of the heavy lifting has been done already. We calculate directly
  \begin{align*}
    \E_t[\exp(Z_{t, a, b})] &= \exp(Z_{t - 1, a, b}) \E_t[\exp(z_{t, a, b})]\\
    &\leq \exp(Z_{t - 1, a, b}) \Phi (w_{t,a}, w_{t,b}, \lambda_{t,a,b}, \alpha_{t,a,b}, \beta_{t,a,b})\\
    &\leq \exp(Z_{t - 1, a, b})~,
  \end{align*}
  where the first equality follows from the definition of $Z_{t, a, b}$, the first inequality follows by Lemma~\ref{lem:safetyconvex}, which we can apply as we assumed that $\Delta_{a, b} \leq 0$ and the second inequality is by assumption on $\Phi$ and the relevant parameters.

  For the second part of the proof, spell out our stopping condition: In order to recommend a wrong action $\atilde \neq \astar$, there must be a timestep $t$ where $\atilde$ beats all other actions $b$ and especially including $\astar$. Thus by direct calculation we can show that
  \begin{align*}
    \Prob \big( \text{$\tau < \infty$ and $\atilde \neq \astar$} \big) &= \Prob \l( \exists \atilde \neq \astar, \; \exists t, \; \forall b \neq \atilde \colon Z_{t, \atilde, b} \geq \log \l( \frac{K - 1}{\delta} \r) \r)\\
    &\leq \Prob \l( \exists \atilde \neq \astar \; \exists t \colon Z_{t, \atilde, \astar} \geq \log \l( \frac{K - 1}{\delta} \r) \r)\\
    &\leq \sum_{\atilde \neq \astar} \Prob \l( \forall t \colon Z_{t, \atilde, \astar} \geq \log \l( \frac{K - 1}{\delta} \r) \r)\\
    &\leq \delta
  \end{align*}
  where the equality is from the stopping condition as outlined in Equation~\ref{eq:stopping.and.recommendation}, the first inequality is relaxing the stopping condition to a sub-optimal action $\atilde$ beating the optimal action $\astar$, as opposed to a sub-optimal action beating all other actions including the optimal action. The second inequality follows by a union bound and the final inequality follows by applying the first half of this proof with $\delta' = (K - 1)/\delta$.
\end{proof}

\begin{restatable}{relemma}{exactphi} \label{lem:exact alpha beta}
  The optimisation problem $D$ from \eqref{eq:Dtab} is maximised given fixed $\lambda \ge 0$ at
  \begin{align*}
    \alpha^*(w_a, w_b, \lambda)
    &~\coloneq~
    \frac{\lambda^{2} \sigma^{2}}{2 w_{a}}
    - w_{a} \log \frac{\left(w_{a} + w_{b}\right) \sinh{\left(\frac{U \lambda}{w_{b}} \right)}}{
      w_{a} \sinh \left(U \lambda \left(
          \frac{1}{w_{a}}
          + \frac{1}{w_{b}}
        \right)
      \right)
    }
  \\
  \beta^*(w_a, w_b, \lambda)
  &~\coloneq~
  \frac{\lambda^{2} \sigma^{2}}{2 w_{b}}
  - w_{b} \log \frac{\left(w_{a} + w_{b}\right) \sinh{\left(\frac{U \lambda}{w_{ a}} \right)}}{w_{ b} \sinh \left(U \lambda \left(
          \frac{1}{w_{a}}
          + \frac{1}{w_{b}}
        \right)
      \right) }
  ,
  \end{align*}
  (where $\sinh(x) \coloneq \frac{e^x-e^{-x}}{2}$) so that in particular
  \begin{align*}
    D(w_a, w_b, \Delta)
    ~=~
    \max_{\lambda \ge 0}~ \lambda \Delta - \alpha^*(w_a, w_b, \lambda) - \beta^*(w_a, w_b, \lambda)
    .
  \end{align*}
\end{restatable}

\begin{proof}
  Fix any $\lambda, w_a, w_b \geq 0$, we aim to find $\alpha$ and $\beta$ optimally. Since $\Phi(w_a, w_b, \lambda, 0, 0) \geq 1$ and $\Phi$ and the objective value is decreasing in both $\alpha$ and $\beta$ the optimal $\alpha$ and $\beta$ will fulfill the safety constraint with equality, that is $\Phi(w_{a}, w_{b}, \lambda, \alpha, \beta) = 1$. Spelling that constraint out with $\Phi$ from Equation~\eqref{eqn:def Phi} gives
  \begin{align*}
    \max_{u \in \{-U, U\}} w_{a} e^{ \frac{\lambda u - \alpha}{w_{a}} + \frac{\lambda^2 \sigma^2}{2 w_{a}^2}} + w_{b} e^{\frac{-\lambda u - \beta}{w_{b}} + \frac{\lambda^2 \sigma^2}{2 w_{b}^2}} = w_a + w_b~.
  \end{align*}
  We will argue that at optimality both cases are saturated, i.e.\
  \begin{subequations}\label{eq:two.u}
  \begin{align}
    w_{a} e^{ \frac{\lambda U - \alpha}{w_{a}} + \frac{\lambda^2}{w_{a}^2} \frac{\sigma^2}{2} } + w_{b} e^{ \frac{-\lambda U - \beta}{w_{b}} + \frac{\lambda^2}{w_{b}^2} \frac{\sigma^2}{2} } &= w_a + w_b\\
    w_{a} e^{ \frac{-\lambda U - \alpha}{w_{a}} + \frac{\lambda^2}{w_{a}^2} \frac{\sigma^2}{2} } + w_{b} e^{ \frac{\lambda U - \beta}{w_{b}} + \frac{\lambda^2}{w_{b}^2} \frac{\sigma^2}{2} } &= w_a + w_b~.
  \end{align}
\end{subequations}
Solving for $\alpha,\beta$ uniquely results in the expressions in the lemma. It remains to check that this is optimal. We will do so by witnessing non-negative Lagrange multipliers for the two constraints \eqref{eq:two.u} and checking the KKT conditions. That is, we solve for $p$ and $q$ so that the weighted gradient sum of \eqref{eq:two.u} cancels the gradient of the objective \eqref{eq:Dtab}, which is $(-1,-1)$. This results in
\[
  p
  ~=~
  \frac{w_a w_b}{w_a+w_b} \left(
    \frac{\frac{1}{w_b}}{1-e^{-\frac{2 \lambda  U}{w_b}}}
    - \frac{\frac{1}{w_a}}{e^{\frac{2 \lambda  U}{w_a}}-1}
  \right)
  \quad\text{and}\quad
  q
  ~=~
  \frac{w_a w_b}{w_a+w_b} \left(
    \frac{\frac{1}{w_a}}{1-e^{-\frac{2 \lambda  U}{w_a}}}
    - \frac{\frac{1}{w_b}}{e^{\frac{2 \lambda  U}{w_b}}-1
    }\right)
  .
\]
  It then remains to check that both $p$ and $q$ are non-negative. As this is symmetric, we do it for $p$. We use that $e^x \ge 1+x$ twice, to find
  \[
    p
    ~\ge~
    \frac{w_a w_b}{w_a+w_b} \left(
    \frac{\frac{1}{w_b}}{\frac{2 \lambda  U}{w_b}}
    - \frac{\frac{1}{w_a}}{\frac{2 \lambda  U}{w_a}}
  \right)
  ~=~
  0.
  \]
\end{proof}

\section{Upper Bound - Maximin Drift - Details}\label{app:drift}

We will show in Lemma~\ref{lem:exact alpha beta} that the optimisation over $\alpha$ and $\beta$ can be performed analytically. In fact, we will show that for each $\lambda \ge 0$, abbreviating $h(x) \coloneq \ln \frac{\sinh x}{x}$ and $W \coloneq \frac{w_a w_b}{w_a + w_b}$, we have
\begin{equation}\label{eq:h2}
  \min_{\substack{
      \alpha, \beta \in \reals
      \\
      \Phi(w_{a}, w_{b}, \lambda, \alpha, \beta) \le 1
    }}
  \alpha + \beta
  =
  \frac{\lambda^2 \sigma^2}{2 W} + h_2(\lambda)
\end{equation}

where $h_2(\lambda) \coloneq (w_a+w_b) h\left(\frac{\lambda U}{W}\right) - w_b h\left(\frac{\lambda  U}{w_a}\right) - w_a h\left(\frac{\lambda  U}{w_b}\right)$.

Further progress is not possible in closed-form. Yet both for interpretation and further analysis of $D(w_a,w_b,\Delta)$ the following sandwich is useful:
\begin{lemma}\label{lem:lambda}
  Let $W \coloneq \frac{1}{\frac{1}{w_a} + \frac{1}{w_b}} = \frac{w_a w_b}{w_a + w_b}$ so that $W \in [0,\frac{1}{4}]$, and define the convex functions $h_1 \le h_2 \le h_3$ by $h_1(\lambda) \coloneq 0$, \eqref{eq:h2} and $h_3(\lambda) \coloneq \frac{\lambda^2 U^2}{2 W}$. For $i \in \setl{1,2,3}$, let $g_i$ and $\lambda^*_i$ denote the value and maximiser of the problem $g_i \coloneq \max_{\lambda \ge 0}~ \lambda \Delta - \frac{\lambda^2 \sigma^2}{2 W} - h_i(\lambda)$. Then
  \[
    \frac{\Delta^2 W }{2 \left(\sigma^2+U^2\right)}
    ~=~ g_3
    ~\le~ g_2
    ~\le~ g_1
    ~=~
    \frac{\Delta^2 W}{2 \sigma^2}
    \qquad
    \text{and}
    \qquad
    \frac{\Delta  W}{\sigma^2+U^2}
    ~=~ \lambda_3^*
    ~\le~ \lambda_2^*
    ~\le~ \lambda_1^*
    ~=~ \frac{\Delta W}{\sigma^2}
    .
  \]
\end{lemma}

\begin{proof}
  The difference $h_2-h_1 = h_2$ is convex, even and zero at zero by inspection, and the difference $h_3-h_2$ is convex, even and zero at zero by Lemma~\ref{lem:WOW}. Hence the values and maximisers are monotonic as stated.
\end{proof}

\begin{lemma}\label{lem:WOW}
  Fix $\theta \in [0,1]$ and let $h(x) \coloneq \ln \frac{\sinh x}{x}$. This function is convex, even, and zero at $x=0$:
\[
  x ~\mapsto~
  \theta (1-\theta) \frac{x^2}{2}
  - \dell{
    h(x)
    - \theta h(\theta x)
    - (1-\theta) h((1-\theta) x)
  }
  .
\]
\end{lemma}
Positivity of this function is in fact Lemma~\ref{lem:Koolen}.

\begin{proof}
  Let $f(x) = \frac{x^2}{6} - h(x)$ and $g(x) = x^2 f''(x) = -\frac{2 x^2}{3}+x^2 \coth^2(x)-1 =
  \sum_{n=1}^\infty \rho\dell{\frac{x}{n \pi}}$ where $\rho(x) = \frac{2 x^4 \left(x^2+3\right)}{\left(x^2+1\right)^2}
  $. The latter function is even and convex (as $\rho$ is), and hence increasing for positive inputs. We hence obtain
  \[
    g(x)
    ~=~
    \theta g(x) + (1-\theta) g(x)
    ~\ge~
    \theta g(\theta x) + (1-\theta) g((1-\theta) x)
  \]
  Rewriting that inequality gives
  \[
    x^2 f''(x)
    ~\ge~
    \theta^3 x^2 f''(\theta x)
    + (1-\theta)^3 x^2 f''((1-\theta) x)
  \]
  Dividing by $x^2$ gives
  \[
    f''(x)
    ~\ge~
    \theta^3 f''(\theta x)
    + (1-\theta)^3 f''((1-\theta) x)
  \]
  thus proving that the following function is convex
  \[
    x ~\mapsto~ f(x) - \theta f(\theta x) - (1-\theta) f((1-\theta) x)
  \]
  which finally unpacks to convexity of
  \[
    x
    ~\mapsto~
    \theta (1-\theta) \frac{x^2}{2}
    - \dell{
      h(x)
      - \theta h(\theta x)
      - (1-\theta) h((1-\theta) x)
    }
    .
  \]
\end{proof}

\section{Upper Bound - Known Gaps - Details}\label{app:power}

\begin{restatable}{relemma}{ismdrift}\label{lem:drift ism}
  For each pair $a \neq b$, the centred evidence $z_{t,a,b} - \ex_t[z_{t,a,b}]$ is sub-Gaussian with variance proxy
    \begin{align*}
      \rho^2
      &~=~
      \frac{ \sigma^2 \lambda_{t, a, b}^2}{ \min \l\{ w_{t, a}, w_{t, b} \r\}^2 }\\
      &\qquad + \max \setl{
        \frac{1}{2} \left(\left| \frac{\alpha_{t,a,b} }{w_{t,a}}-\frac{\beta_{t,a,b} }{w_{t,b}}\right| +\lambda_{t,a,b}  U \left(\frac{1}{w_{t,a}}+\frac{1}{w_{t,b}}\right)\right)
        ,
        \frac{\lambda_{t,a,b}  U}{w_{t,a}},
        \frac{\lambda_{t,a,b}  U}{w_{t,b}}
      }^2.
    \end{align*}
\end{restatable}
\begin{proof}
  Let $D \coloneq \E_t \l[ z_{t, a, b}\r]$. We first define the shift-conditional drift
  \begin{equation}\label{eq:q}
    q_{t, a, b}
    ~\coloneq~ \E_t \sbrcl{ z_{t, a, b} }{B_t}
    ~=~
    \frac{\indicator_{t, a}}{w_{t, a}} \cdot \dell{
      \lambda_{t,a,b} \mu_{t,a}  - \alpha_{t,a,b}
    }
    + \frac{\indicator_{t, b}}{w_{t, b}} \cdot \dell{
      -\lambda_{t,a,b} \mu_{t,b}- \beta_{t,a,b}
    }~,
  \end{equation}
  which we then add and subtract
  \begin{align*}
    \E_t \l[ \exp \l( \xi D - \xi z_{t, a, b} \r) \r]
    &= \E_t \l[ \exp \l( \xi D - \xi q_{t, a, b} + \xi q_{t, a, b} - \xi z_{t, a, b} \r)  \r] \numberthis \label{eqn:D q decomp}\\
    &= \E_t \l[ \exp \l( \xi D - \xi q_{t, a, b} \r)  \E_t \l[ \exp \l( \xi q_{t, a, b} - \xi z_{t, a, b} \r) \given* A_t, B_t \r] \r]~.
  \end{align*}
  We bound the factors separately, starting with innermost conditional expectation. Combining \eqref{eq:z} and \eqref{eq:q}, we find that the shift-conditional centred evidence is
  \begin{align*}
    z_{t, a, b} - q_{t, a, b}
    ~=~
    \begin{cases}
      \frac{\lambda_{t, a, b} (X_t - \mu_{t,a})}{w_{t, a}}, & \text{if } A_t = a,\\
      \frac{- \lambda_{t, a, b} (X_t - \mu_{t,b})}{w_{t, b}}, & \text{if } A_t = b,\\
      0,  & \text{otherwise}
    \end{cases}~.
  \end{align*}
  Given $A_t$ and $B_t$ (in addition to $H_{t-1}$), this is a fixed scaling of a centred $\sigma^2$-sub-Gaussian random variable. Taking worst case over the scaling, we hence find that
  \begin{align*}
    \E_t \l[ \exp \l( \xi q_{t, a, b} - \xi z_{t, a, b} \r) \given* A_t, B_t \r]
    &\leq \exp \l( \frac{\xi^2 \sigma^2  \lambda_{t, a, b}^2}{2 \min \l\{ w_{t, a}, w_{t, b} \r\}^2 } \r)~, \numberthis \label{eqn:q z bound}
  \end{align*}
  For the other factor we use Hoeffding's Inequality \citep[Lemma~A.1]{cesa-bianchi2006plg}, which states that for any bounded random variable $Y \in [x, y]$
  \begin{align*}
      \E [\exp(\xi Y)] \leq \exp \l( \xi \E[Y] + \frac{\xi^2 (x - y)^2}{8} \r)~,
    \end{align*}
    Recall that the shift-conditional means $\mu_{t,a}$ and $\mu_{t,b}$ are bounded by $[\pm U]$. So the tightest range on $q_{t,a,b}$ is of radius
    \begin{align*}
      U'
      &~\coloneq~
      \min_{c \in \reals} \max_{u \in \setl{\pm U}} \max \setl{ \absl{\frac{\lambda_{t,a,b} u - \alpha_{t,a,b}}{w_{t,a}} - c}, \absl{\frac{- \lambda_{t,a,b} u - \beta_{t,a,b}}{w_{t,b}}-c}}
          \\
      &~=~
      \max \setl{
        \frac{1}{2} \left(\left| \frac{\alpha_{t,a,b} }{w_{t,a}}-\frac{\beta_{t,a,b} }{w_{t,b}}\right| +\lambda_{t,a,b}  U \left(\frac{1}{w_{t,a}}+\frac{1}{w_{t,b}}\right)\right)
        ,
        \frac{\lambda_{t,a,b}  U}{w_{t,a}},
        \frac{\lambda_{t,a,b}  U}{w_{t,b}}
      }
\end{align*}
from which we get
  \[ \Et \l[ \exp \l( \xi q_{t, a, b} - \xi D \r) \r]\\
    ~\leq~ \exp \l(
    \frac{\xi^2}{2} U'^2
    \r).
  \]
  Together with Equations~\eqref{eqn:D q decomp} and ~\eqref{eqn:q z bound} we have
  \begin{align*}
    \E_t \l[ \exp \l( \xi D - \xi z_{t, a, b} \r)  \r]
    &\leq \exp \l( \frac{\xi^2 \sigma^2 \lambda_{t, a, b}^2}{2 \min \l\{ w_{t, a}^2, w_{t, b}^2 \r\} } \r) \cdot
    \exp \l(
    \frac{\xi^2}{2} U'^2
    \r)\\
    &= \exp \l( \frac{\xi^2}{2} \l(
      \frac{ \sigma^2 \lambda_{t, a, b}^2}{ \min \l\{ w_{t, a}^2, w_{t, b}^2 \r\} }
      + U'^2
    \r) \r)~,
  \end{align*}
  as required.
\end{proof}

\oracle*

\begin{proof}
  \newcommand{\dif}{\mathrm{d}}
  Throughout, let us abbreviate $C := \logdeltaK$. By the layer cake representation, we have
\[
  \ex[\tau]
  ~=~
  \int_{0}^\infty \pr\setl{\tau > t} \dif t
\]
For possibly non-integer $t$, with $\floor{t}$ its floor, and for any $\xi > 0$,
\begin{align*}
  \pr \setl{\tau > t}
  &~\le~
  \pr \setl{\exists b \neq a^*: \sum_{s=1}^{\floor{t}} z_s\ups{\astar, b} < C}
  \\
  &~\le~
  \sum_{b \neq a^*} \pr \setl{\sum_{s=1}^{\floor{t}} z_s\ups{\astar, b} < C}
  \\
  &~=~
  \sum_{b \neq a^*} \pr \setl{e^{\xi \sum_{s=1}^{\floor{t}} (D -  z_s\ups{\astar, b})} > e^{\xi(\floor{t} D - C)}}
  \\
  &~\le~
  (K-1) e^{\frac{1}{2} \floor{t} \xi^2 \rho^2 - \xi(\floor{t} D - C)}
  \\
  &~\le~
  (K-1) e^{\frac{1}{2} (t-1) \xi^2 \rho^2 - \xi((t-1) D - C)}
  ,
\end{align*}
where the inequalities are enlarging the event of the stopping rule \eqref{eq:stopping.and.recommendation} not firing, a union bound over arms, sub-Gaussianty and integer rounding. As also trivially $\pr\setl{\tau > t} \le 1$, we conclude that
\begin{align*}
  \ex[\tau]
  &~\le~
  \int_0^\infty
  \min \setl{1,
    (K-1) e^{\frac{1}{2} (t-1) \xi^2 \rho^2 - \xi((t-1) D - C)}
  }
  \dif t
  \\
  &~=~
  1 + \frac{C+\frac{1+\ln (K-1)}{\xi }}{D -\frac{\xi  \rho ^2}{2}}
  .
\end{align*}
The optimal choice for $\xi$, found by cancelling the derivative, is
\[
  \xi
  ~=~
  \frac{\frac{\sqrt{(1 + \ln (K-1)) \left(2 C D +\rho ^2 (1 + \ln (K-1))\right)}}{\rho }-(1 + \ln (K-1))}{C}
\]
which is indeed positive. Plugging it in, the bound becomes
\[
  1
  +
  \left(\sqrt{\frac{C}{D }+\frac{\rho ^2}{2 D ^2} (1 + \ln (K-1))}+\sqrt{\frac{\rho ^2}{2 D ^2} (1 + \ln (K-1))}\right)^2
\]
as desired.
\end{proof}

\section{Upper Bound - Interpretable Bounds - Details}\label{app:interpretable}

\lemESbound*

\begin{proof}
  Let $\theta_b = \frac{w_{\astar}}{w_b + w_{\astar}}$ and $\xi_b = \frac{w_{\astar}w_b}{w_b + w_{\astar}}$. By Lemma~\ref{lem:exact alpha beta} we have that  
  \begin{alignat*}{2}
      D^*(\bm \Delta) &= \mathrlap{\max_w \min_{b \neq {\astar}} \max_{\lambda > 0}  D(w_{\astar}, w_b, \Delta_b)}\\
      & = \max_w \min_{b \neq {{\astar}}} \max_{\lambda \ge 0} \mathrlap{\lambda \Delta_b - \alpha^*(w_{\astar}, w_b, \lambda) - \beta^*(w_{\astar}, w_b, \lambda)}\\
      & = \max_w \min_{b \neq {\astar}} \max_{\lambda \ge 0} \lambda \Delta_b && - \frac{\lambda^{2} \sigma^{2}}{2 w_{\astar}}
  + w_{\astar} \log \frac{w_{\astar} + w_{b}}{w_{\astar}} \frac{\sinh{\left(\frac{U \lambda}{w_{b}} \right)}}{
    \sinh \left(U \lambda \left(
        \frac{1}{w_{\astar}}
        + \frac{1}{w_{b}}
      \right)
    \right)} \\
    &&& - \frac{\lambda^{2} \sigma^{2}}{2 w_{b}}
+ w_{b} \log \frac{w_{\astar} + w_{b}}{w_{b}} \frac{\sinh{\left(\frac{U \lambda}{w_{\astar}} \right)}}{\sinh \left(U \lambda \left(
        \frac{1}{w_{\astar}}
        + \frac{1}{w_{b}}
      \right)
    \right) } \\
      & = \max_w \min_{b \neq {\astar}} \max_{\lambda \ge 0} \lambda \Delta_b && - \frac{\lambda^{2} \sigma^{2}}{2 w_{\astar}}
  + w_{\astar} \log \frac{1}{\theta_b} \frac{\sinh{\left(\frac{U \lambda}{w_{b}} \right)}}{
    \sinh \left(U \lambda \left(
        \frac{1}{w_{\astar}}
        + \frac{1}{w_{b}}
      \right)
    \right)} \\
    &&& - \frac{\lambda^{2} \sigma^{2}}{2 w_{b}}
+ w_{b} \log \frac{1}{1 - \theta_b} \frac{\sinh{\left(\frac{U \lambda}{w_{\astar}} \right)}}{\sinh \left(U \lambda \left(
        \frac{1}{w_{\astar}}
        + \frac{1}{w_{b}}
      \right)
    \right) }~.
  \end{alignat*}
  By reparametrizing $\lambda = \xi_b \eta$ for $\eta > 0$ we have 
  \begin{align*}
    D^*(\bm \Delta) = &\max_w \min_{b \neq {\astar}} \max_{\eta \ge 0} \xi_b \eta \Delta - \xi_b (1 - \theta_b)\frac{\eta^{2} \sigma^{2}}{2 } - \xi_b \theta_b\frac{\eta^{2} \sigma^{2}}{2}\\
    & \quad + (w_{\astar} + w_b)\Big( \theta_b \log \frac{\sinh{\left(\theta_b{U \eta} \right)}}{\theta_b
    \sinh \left(U \eta
    \right)} + (1-\theta_b) \log \frac{\sinh{\left({U \eta (1-\theta_b)} \right)}}{ \left(1-\theta_b\right) \sinh \left(U \eta
    \right) }\Big)~. \numberthis \label{eq:reparDstar}
  \end{align*}
  By Lemma~\ref{lem:Koolen}, we have
  \begin{align}\label{eq:sinhinproof}
      \theta_b \log \frac{\sinh{\left(\theta_b{U \eta} \right)}}{\theta_b
    \sinh \left(U \eta
    \right)} + (1-\theta_b) \log \frac{\sinh{\left({U \eta (1-\theta_b)} \right)}}{ \left(1-\theta_b\right) \sinh \left(U \eta
    \right) } \geq - \theta_b(1 - \theta_b) \frac{(U \eta)^2}{2}\,.
  \end{align}
  Thus, 
  \begin{align*}
      & D^*(\bm \Delta) \\
    & \geq \max_w \min_{b \neq {\astar}} \max_{\eta \ge 0} \xi_b \eta \Delta_b - \xi_b (1 - \theta_b)\frac{\eta^{2} \sigma^{2}}{2 } - \xi_b \theta_b\frac{\eta^{2} \sigma^{2}}{2} - (w_{\astar} + w_b)\theta_b(1 - \theta_b) \frac{\eta^2 U^2}{2}
    \\
    & = \max_w \min_{b \neq {\astar}} \max_{\eta \ge 0} \xi_b\Big( \eta \Delta_b - \eta^{2}\frac{ \sigma^{2} + U^2}{2}\Big) \\
    & = \max_w \min_{b \neq {\astar}} \frac{w_{\astar}w_b}{w_{\astar} + w_b}\frac{\Delta_b^2}{2(\sigma^{2} + U^2)}~,
  \end{align*}
  and we obtain the first result after recalling that $b^* = \argmin_{b \in \Acal \setminus \{\astar\}} \Delta_b$ is the second best action and following the same steps as Appendix~A.4 of \citet{garivier16optimalBAI}, to find that 
  \begin{align*}
      \max_w \min_{b \neq {\astar}} \frac{w_{\astar}w_b}{w_{\astar} + w_b}\frac{\Delta_b^2}{2(\sigma^{2} + U^2)} \geq \frac{1}{2}\l(\frac{2\big(U^2 + \sigma^2\big)}{\Delta_{{b^*}}^2} + \sum_{a \in [K]\setminus \astar}\frac{2\big(U^2 + \sigma^2\big)}{\Delta_a^2}\r)^{-1}~.
  \end{align*}

  For the second part of the statement we commence by upper bounding \eqref{eq:sinhinproof}. By Lemma~\ref{lem:Koolen}, with $\gamma = 3(\coth 1 - 1) \approx 0.939$ as defined there, for $\eta$ such that $\eta U \leq 1$ we have that
  \begin{align*}
      - \theta_b \log \frac{\sinh{\left(\theta_b{U \eta} \right)}}{\theta_b
    \sinh \left(U \eta
    \right)} - (1-\theta_b) \log \frac{\sinh{\left({U \eta (1-\theta_b)} \right)}}{ \left(1-\theta_b\right) \sinh \left(U \eta
    \right) } \geq \theta_b(1 - \theta_b) \frac{\gamma (U \eta)^2}{2}\,,
  \end{align*}
  and for $\eta$ such that $\eta U > 1$ we have that
  \begin{align*}
      - \theta_b \log \frac{\sinh{\left(\theta_b{U \eta} \right)}}{\theta_b
    \sinh \left(U \eta
    \right)} - (1-\theta_b) \log \frac{\sinh{\left({U \eta (1-\theta_b)} \right)}}{ \left(1-\theta_b\right) \sinh \left(U \eta
    \right) } \geq \theta_b(1 - \theta_b) \frac{ \gamma\, \eta U }{2}\,.
  \end{align*}
  Thus, continuing from \eqref{eq:reparDstar}, for $2 \max_b \Delta_b \leq \gamma U$, we have that
  \begin{align*}
  D^*(\bm \Delta)
    & \leq \max_w \min_{b \neq {\astar}} \max\Big\{ \max_{\eta \in [0, \frac1U]} \xi_b \eta \Delta_b - \xi_b \frac{\eta^{2} \big(\sigma^{2} + \gamma U^2\big)}{2 } \,,~
    \max_{\eta > \frac1U} \xi_b \eta \big(\Delta_b - \tfrac{\gamma}{2}U\big) - \xi_b \frac{\eta^{2} \sigma^{2}}{2 } \Big\}\\
    & \leq \max_w \min_{b \neq {\astar}}  \frac{w_{\astar}w_b}{w_{\astar} + w_b}\frac{\Delta_b^2}{2\big(\sigma^{2} + \gamma U^2\big)} \,,
  \end{align*}
  where we used that $\frac{\Delta_b}{\gamma U^2 + \sigma^2} \leq \frac{1}{U}$, so that the first inner maximum is attained inside $[0, \frac1U]$, and that the second inner maximum is non-positive. We finish again by following the same steps as Appendix~A.4 of \citet{garivier16optimalBAI}, to find that 
  \begin{align*}
    \max_w \min_{b \neq {\astar}} \frac{w_{\astar}w_b}{w_{\astar} + w_b}\frac{\Delta_b^2}{2(\sigma^{2} + \gamma U^2)} \leq \l(\frac{2\big(\gamma U^2 + \sigma^2\big)}{\Delta_{{b^*}}^2} + \sum_{a \in [K]\setminus \astar}\frac{2\big(\gamma U^2 + \sigma^2\big)}{\Delta_a^2}\r)^{-1}~.
  \end{align*}
  
\end{proof}

\begin{lemma}\label{lem:Koolen}
Let $\psi(x) \coloneq \ln \frac{\sinh x}{x}$ with $\psi(0) \coloneq 0$, and let $\gamma \coloneq 3(\coth 1 - 1) \approx 0.939$. For all $\theta \in [0,1]$ and $x \ge 0$,
\[
  \tfrac{\gamma}{2}\, \theta(1-\theta)\, \min\{x, x^2\}
  ~\le~
  \psi(x) - \theta\, \psi(\theta x) - (1-\theta)\, \psi((1-\theta) x)
  ~\le~
  \tfrac{1}{2}\, \theta(1-\theta)\, x^2
  ~.
\]
\end{lemma}
\begin{proof}
Let $L(t) \coloneq \psi'(t) = \coth t - \frac{1}{t}$. By Euler's partial fraction expansion $\coth t = \frac{1}{t} + 2t \sum_{k=1}^\infty \frac{1}{\pi^2 k^2 + t^2}$ we have
\[
  \frac{L(t)}{t} ~=~ 2 \sum_{k=1}^\infty \frac{1}{\pi^2 k^2 + t^2}~,
\]
which is positive and decreasing in $t > 0$ with limit $\frac{1}{3}$ at $t = 0$, and $L$ is increasing since $L'(t) = \frac{1}{t^2} - \frac{1}{\sinh^2 t} > 0$. Now fix $0 \le s \le x$. Writing the integrand as $L(t) = t \cdot \frac{L(t)}{t}$ and using $\frac{L(x)}{x} \le \frac{L(t)}{t} \le \frac{1}{3}$ for all $t \in (0, x]$, we obtain
\[
  \frac{L(x)}{x} \cdot \frac{x^2 - s^2}{2}
  ~=~ \frac{L(x)}{x} \int_s^x t \, d t
  ~\le~
  \underbrace{\int_s^x t \cdot \frac{L(t)}{t} \, d t}_{\psi(x) - \psi(s)}
  ~\le~
  \frac{1}{3} \int_s^x t \, d t
  ~=~ \frac{x^2 - s^2}{6}~,
\]
where the underbrace follows as $L = \psi'$.
Writing the middle expression of the lemma as $\theta\big(\psi(x) - \psi(\theta x)\big) + (1-\theta)\big(\psi(x) - \psi((1-\theta)x)\big)$, and using the above display twice, once with $s = \theta x$ and once with $s = (1-\theta) x$, we obtain
\begin{align*}
  \frac{L(x)}{x} \cdot \frac{x^2}{2} \Big(\theta (1 - \theta^2) + (1-\theta)\big(1 - (1-\theta)^2\big)\Big)
  &~\le~
  \psi(x) - \theta\, \psi(\theta x) - (1-\theta)\, \psi((1-\theta) x) \\
  &~\le~
  \frac{x^2}{6} \Big(\theta (1 - \theta^2) + (1-\theta)\big(1 - (1-\theta)^2\big)\Big)~.
\end{align*}
Since
\[
  \theta (1 - \theta^2) + (1-\theta)\big(1 - (1-\theta)^2\big) ~=~ \theta(1-\theta)\big((1+\theta) + (2-\theta)\big) ~=~ 3\, \theta(1-\theta)~,
\]
this gives
\[
  \tfrac{3}{2}\, \theta(1-\theta)\, x L(x)
  ~\le~
  \psi(x) - \theta\, \psi(\theta x) - (1-\theta)\, \psi((1-\theta) x)
  ~\le~
  \tfrac{1}{2}\, \theta(1-\theta)\, x^2
  ~.
\]
The right-hand side is the claimed upper bound. For the lower bound, since $L(x)/x$ is decreasing we have $x L(x) \ge L(1)\, x^2$ for $x \le 1$, and since $L$ is increasing we have $x L(x) \ge L(1)\, x$ for $x \ge 1$. Hence $x L(x) \ge L(1) \min\{x, x^2\}$, and $\tfrac{3}{2} L(1) = \tfrac{\gamma}{2}$.
\end{proof}

\section{Lower Bounds - Details}\label{app:lower bounds}

\lemklinflate*
\begin{proof}
For $\triangle \in \reals$, let $q_\triangle$ denote the density of $\triangle + \epsilon + B'$, so that the claim reads $\KL{q_0}{q_\triangle} \le \triangle^2 / (2 (\sigma^2 + R^2/\pi^2))$ for all $|\triangle| \le d$. For a random variable $X$ with absolutely continuous density $q$, write $I(X) \coloneq \int q'(x)^2 / q(x) \, d x$ for its Fisher information for location. The density $q_0$ is the convolution of $p$ with a centred Gaussian, hence a Gaussian mixture with means in $[-R, R]$: it is smooth, strictly positive, and has Gaussian tails. The location family $\{q_\triangle\}_{\triangle \in \reals}$ therefore satisfies the regularity conditions of \citet[Definition~7.21]{polyanskiy2025information}, exactly as a Gaussian location family does, and \citet[Theorem~7.22]{polyanskiy2025information}, applied to $\triangle \mapsto q_\triangle$ and to $\triangle \mapsto q_{-\triangle}$, yields
\begin{equation}\label{eqn:klfisher}
  \lim_{\triangle \to 0} \frac{\KL{q_0}{q_\triangle}}{\triangle^2 / 2} ~=~ I(\epsilon + B')~.
\end{equation}
We next bound this Fisher information. We have $I(\epsilon) = 1/\sigma^2$. For $B'$, let $G$ have density $g(u) = \cos^2(\pi u / 2)$ on $[-1, 1]$, so that $R \, G \sim p$. This $g$ has Fisher information $I(G) = \pi^2$, which is the smallest Fisher information of any density supported on $[-1, 1]$ \citep[Equation~(29.13)]{polyanskiy2025information}. By the scaling rule of Fisher information \citep[Equation~(2.35)]{polyanskiy2025information}, we obtain
\[
  I(B') ~=~ I(R \, G) ~=~ \frac{I(G)}{R^2} ~=~ \frac{\pi^2}{R^2}~.
\]
The Fisher information inequality \citep[Exercise~I.47]{polyanskiy2025information}, states that for independent $X$ and $Y$
\[
  \frac{1}{I(X + Y)} ~\ge~ \frac{1}{I(X)} + \frac{1}{I(Y)}~,
\]
with equality only if both $X$ and $Y$ are Gaussian \citep[Section~II-C4]{rioul2011information}. As $B'$ is bounded, we thus have the strict inequality
\begin{equation}\label{eqn:stam}
  I(\epsilon + B') ~<~ \frac{1}{\sigma^2 + R^2/\pi^2}~.
\end{equation}
Combining~\eqref{eqn:klfisher} and~\eqref{eqn:stam}, there is $d > 0$ such that
\[
  \KL{q_0}{q_\triangle} ~\le~ \frac{\triangle^2}{2 (\sigma^2 + R^2/\pi^2)}
  \qquad \text{for all } |\triangle| \le d~,
\]
which is the claim.
\end{proof}

\thmlowerbound*

\begin{proof}
  Fix any $\delta$-correct algorithm and let $\bm \Delta$ be gaps with optimal arm $\astar$ as in the statement of the theorem, that is, with $R \coloneq U - \max_a \Delta_a / 2 > 0$, the constant $d > 0$ of Lemma~\ref{lem:kl inflate} for this $\sigma^2$ and $R$ satisfies $\max_a \Delta_a \le d$. Let $p$ be the shift distribution of Equation~\eqref{eqn:hardshiftdist} with this $R$, and let $\lband$ be the hard shift instance for $\bm \Delta$, that is, $B_t = \max_a \Delta_a/2 + B'_t$ with $B'_t \sim p$ i.i.d. Since $B'_t \in [-R, R]$, we have $B_t \in [\max_a \Delta_a/2 - R,\, \max_a \Delta_a/2 + R] = [\max_a \Delta_a - U, U]$ and hence $|B_t - \Delta_a| \le U$ for every arm $a$, so $\lband$ fulfills Assumption~\ref{ass:main}.

  Then, our lower bound construction follows \cite{garivier16optimalBAI} and thus we now introduce an alternative set of gaps $\tilde{\bm \Delta}$ %
  with a different best arm, that is $\argmin_a \tilde{\Delta}_a \neq \astar$, and let $\lalt$ be the instance with gaps $\tilde{\bm \Delta}$ and the same shifts $B_t$. Note that $\tilde{\bm \Delta}$ does not need to be a real gap vector, i.e., it does not need to contain an action with gap of 0. We then consider the KL of the entire history when playing against $\lband$ compared to the entire history playing against $\lalt$. The idea of the proof is to bound this KL from both sides, the KL is given as
  \begin{align}\label{eqn:big kl}
    \KL{ \Prob_{\lband} \{ A_1, X_1, \ldots, A_\tau, X_\tau, \tau, \atilde \} }{ \Prob_{\lalt} \{ A_1, X_1, \ldots, A_\tau, X_\tau, \tau, \atilde \} }~,
  \end{align}
  where we recall that $\tau$ is the stopping time and $\atilde$ is the arm the algorithm ultimately recommends. We then lower bound that KL with 
  \begin{align*}
    &\quad \KL{ \Prob_{\lband} \{ A_1, X_1, \ldots, A_\tau, X_\tau, \tau, \atilde \} }{ \Prob_{\lalt} \{ A_1, X_1, \ldots, A_\tau, X_\tau, \tau, \atilde \} }\\
    &\geq \KL{ \Prob_{\lband} \{ \atilde \neq \astar \} }{ \Prob_{\lalt} \{ \atilde \neq \astar \} }\\
    &\geq \kl{\delta}{1 - \delta}~, \numberthis \label{eqn:KL lower}
  \end{align*}
  where we used the data processing inequality in the first inequality to cut out all variables except $\atilde$ and the second inequality follows as the algorithm is $\delta$-correct on both of $\lband$ and $\lalt$ and the fact that $\astar \neq \argmin_a \tilde{\Delta}_a$ is not the best arm for $\tilde{\bm \Delta}$.
  For the upper bound we start by some direct calculation employing that conditioned on the value of $A_t = a$, we have $X_t = B_t - \Delta_a + \epsilon_t$ with $\epsilon_t \sim \Ncal(0, \sigma^2)$,
  \begin{align*}
    &\quad \KL{ \Prob_{\lband} \{ A_1, X_1, \ldots, A_\tau, X_\tau, \tau, \atilde \} }{ \Prob_{\lalt} \{ A_1, X_1, \ldots, A_\tau, X_\tau, \tau, \atilde \} }\\
    &= \mathop{\E_{\lband}} \l[ \log  \frac{\Prob_{\lband}\{A_1, X_1, \ldots, A_\tau, X_\tau, \tau, \atilde\}}{\Prob_{\lalt}\{A_1, X_1, \ldots, A_\tau, X_\tau, \tau, \atilde\}}  \r]\\
    &= \mathop{\E_{\lband}} \l[ \log  \prod_{t = 1}^\tau \frac{\Prob_{\lband}\{X_t \mid A_t, H_{t - 1}\}}{\Prob_{\lalt}\{X_t \mid A_t, H_{t - 1}\}}  \r]\\
    &= \sum_{t = 1}^\infty \sum_{a \in [K]}  \mathop{\E_{\lband}} \l[   \indicator[A_t = a, \tau \ge t] \cdot \ex_{\lband} \sbrcl{
      \log  \frac{\Prob_{\lband}\{X_t \mid A_t=a, H_{t - 1}\}}{\Prob_{\lalt}\{X_t \mid A_t=a, H_{t - 1}\}} }{A_t=a, H_{t-1}} \r]\\
    &= \sum_{t = 1}^\infty \sum_{a = 1}^K \Prob_{\lband}\{ A_t = a, \tau \geq t \} \KL{B_t - \Delta_a + \epsilon_t}{B_t - \tilde{\Delta}_{a} + \epsilon_t}~, \numberthis \label{eqn:KL upper half way}
  \end{align*}
  where we first used the definition of the KL, then cancelled the actions in the log as their distribution given the history (i.e.\ the sampling rule) is identical in $\lband$ and $\lalt$. The third equality is by enumerating the possible values for $A_t$, the tower rule, pulling the sum outside of the log. The final equality comes from recognizing the inside expectation as a one-round KL that is independent of time. Since $B_t = \max_a \Delta_a/2 + B'_t$ and the KL divergence is invariant under a common translation of both arguments, this one-round KL equals $\KL{\epsilon_t + B'_t}{(\Delta_a - \tilde \Delta_a) + \epsilon_t + B'_t}$. We now plan to apply Lemma~\ref{lem:kl inflate} to it with $\triangle = \Delta_a - \tilde \Delta_a$, where we simply assume that $|\Delta_a - \tilde \Delta_a| \leq \max_a \Delta_a \le d$ for now, which we will verify later. That gives the upper bound for the KL of~\eqref{eqn:big kl}, which we use together with the lower bound of~\eqref{eqn:KL lower} to find that
  \begin{align*}
    \kl{\delta}{1 - \delta}
    &~\leq~  \sum_{a = 1}^K \sum_{t = 1}^\infty \Prob_{\lband}\{A_t = a, \tau \geq t \} \frac{(\Delta_a - \tilde{\Delta}_a)^2}{2 (\sigma^2 + R^2/\pi^2)}.
    \\
    &~=~ \ex_{\lband}\sbrl{\tau }  \sum_{a = 1}^K \frac{\sum_{t = 1}^\infty \Prob_{\lband}\{A_t = a, \tau \geq t \}}{\ex_{\lband}\{\tau \}} \frac{(\Delta_a - \tilde{\Delta}_a)^2}{2 (\sigma^2 + R^2/\pi^2)}~.
  \end{align*}
  using that $\sum_{a \in [K]} \sum_{t = 1}^\infty \Prob_{\lband}\{A_t = a, \tau \geq t \} = \sum_{t = 1}^\infty \Prob_{\lband}\{\tau \geq t \} = \ex_{\lband}[\tau]$.
  We further focus on the RHS of the above equation and worst-case it over all legal alternative gaps $\tilde{\bm \Delta}$, and give the learner the benefit of the doubt, to find
  \begin{align*}
    \kl{\delta}{1 - \delta}
    \leq &\E_{\lband}[\tau] \sup_{\bm w \in \triangle_K} \min_{\substack{\tilde{\bm\Delta} \\  \argmin_a \tilde{\Delta}_a \neq \astar}} \sum_{a = 1}^K w_a \frac{(\Delta_a - \tilde{\Delta}_a)^2}{2 (\sigma^2 + R^2/\pi^2)}\\
    = &\E_{\lband}[\tau] \sup_{\bm w \in \triangle_K} \min_{m, \; b \neq \astar} w_{\astar} \frac{m^2}{2 (\sigma^2 + R^2/\pi^2)} + w_{b} \frac{(\Delta_b - m)^2}{2 (\sigma^2 + R^2/\pi^2)}\\
    = &\E_{\lband}[\tau] \sup_{\bm w \in \triangle_K} \min_{b \neq \astar} \frac{w_{\astar} w_{b}}{w_{\astar} + w_{b}} \frac{\Delta_b^2}{2 (\sigma^2 + R^2/\pi^2)}~,
  \end{align*}
  where we first upper bounded the probability to play an arm by maximizing instead, then recognized that the best alternative will only change the best arm and a second arm $b$, and move them to meet at a midpoint $m$. The last equality is by optimizing over $m = \frac{w_b \Delta_b}{w_{\astar} + w_b}$, where now can also verify that $|\Delta_a - \tilde \Delta_a| = |\Delta_b - m| \leq \Delta_b \leq \max_a \Delta_a$ and thus the application of Lemma~\ref{lem:kl inflate} above was appropriate. By following the same argument as by \citet[Appendix~A.4]{garivier16optimalBAI}, one obtains that
  \begin{align*}
    \sup_{\bm w \in \triangle_K} \min_{b \neq \astar} \frac{w_{\astar} w_{b}}{w_{\astar} + w_{b}} \frac{\Delta_b^2}{2 (\sigma^2 + R^2/\pi^2)} \leq \frac{1}{2(\sigma^2 + R^2/\pi^2)} \l( \frac{1}{\Delta_{{b^*}}^2} + \sum_{a \in [K]\setminus \astar}\frac{1}{\Delta_a^2} \r)^{-1}~.
  \end{align*}
  By rearranging the pieces, we obtain $\E[\tau] \geq 2 (\sigma^2 + R^2/\pi^2) \l( \frac{1}{\Delta_{{b^*}}^2} + \sum_{a \in [K]\setminus \astar}\frac{1}{\Delta_a^2} \r) \kl{\delta}{1-\delta}$, and the claim follows from $\lim_{\delta \to 0} \kl{\delta}{1-\delta} / \logdelta = 1$ and $R = U - \max_a \Delta_a / 2$.
\end{proof}

\section{Concentration: Bernstein with bounded means and sub-Gaussian noise}
We are interested in concentration of time-averages of the importance-weighted reward:
\[
  \dell{
    \frac{\indicator_{t, a} }{w_{t, a}} - \frac{\indicator_{t, b}}{w_{t, b}}
  } X_t
  .
\]

To bound the exponential moment we start with a definition.

\begin{definition}
Fix $R > 0$, $U > 0$ and $\sigma^2 > 0$. Define
\[
  \Xi(R, U, \sigma^2)
  ~\coloneq~
  \max_{r \in [\pm R]} \max_{x \in [\pm U]}~
  \frac{e^{r x + \frac{1}{2} r^2 \sigma^2} - 1 - r x}{r^2}
\]
\end{definition}

\begin{lemma}\label{lem:Xi}
  Fix $R > 0$, $U > 0$ and $\sigma^2 > 0$. Then
  \[
    \Xi(R, U, \sigma^2)
    ~=~
    \frac{e^{R U + \frac{1}{2} R^2 \sigma^2} - 1 - R U}{R^2}
  \]
\end{lemma}
For interpretation sake, we observe that for small $R$ (which is typically under our control),
\[
  \frac{e^{R U + \frac{1}{2} R^2 \sigma^2} - 1 - R U}{R^2}
  ~\le~
  \frac{1}{2} (1 + R U) \left(U^2+\sigma^2\right)
  + O(R^2)
\]
As we will see below, it is not important that $\Xi(R, U, \sigma^2)$ is a small multiple of $U^2+\sigma^2$. What matters is that is constant over rounds.

\begin{proof}[Proof of Lemma~\ref{lem:Xi}]
  The objective is convex in $x$, so the maximiser is at the boundary. By considering both cases, and then by checking that the remaining function is increasing in $r \ge 0$, we find
  \[
    \frac{e^{r x + \frac{1}{2} r^2 \sigma^2} - 1 - r x}{r^2}
    ~\le~
    \frac{e^{|r| U + \frac{1}{2} r^2 \sigma^2} - 1 - |r| U}{r^2}
    ~\le~
    \frac{e^{R U + \frac{1}{2} R^2 \sigma^2} - 1 - R U}{R^2}
    .
  \]
\end{proof}
With that result under our belt, we can now bound the exponential moment.

\begin{lemma}\label{lem:MGFbound}
  In every round, for $\eta$ such that $\sup_a \frac{|\eta|}{w_{t,a}} \le C$ for some $C > 0$, 
  \[
    \ln \ex_t \sbrl{
      e^{\eta \dell{
          \frac{\indicator_{t, a} }{w_{t, a}} - \frac{\indicator_{t, b}}{w_{t, b}}
        } X_t }
    }
    ~\le~
    \eta \Delta_{a,b}
    + \eta^2  \dell{
      \frac{1}{w_{t, a}}
      + \frac{1}{w_{t, b}}
    } \Xi(C, U, \sigma^2)
    .
  \]
\end{lemma}

\begin{proof}
Well,
\begin{align*}
  &\ex \sbrcl{
    e^{\eta \dell{
          \frac{\indicator_{t, a} }{w_{t, a}} - \frac{\indicator_{t, b}}{w_{t, b}}
        } X_t }
  }{B_t}
  \\
  &~=~
  w_{t,a} \ex \sbrcl{
    e^{\eta \dell{
        \frac{X_t}{w_{t, a}}
        }}
  }{A_t=a,B_t}
  + w_{t,b} \ex \sbrcl{
    e^{\eta \dell{
        - \frac{X_t}{w_{t, b}}
        }}
  }{A_t=b,B_t}
  + (1-w_{t,a}-w_{t,b})
  \\
  &~\le~
  w_{t,a} e^{\eta \dell{
      \frac{B_t-\Delta_a}{w_{t, a}}
      } + \frac{\eta^2 \sigma^2}{2 w_{t,a}^2}
  }
  + w_{t,b} e^{\eta \dell{
      - \frac{B_t - \Delta_b}{w_{t, b}}
      }
    + \frac{\eta^2 \sigma^2}{2 w_{t,b}^2}
  }
  + (1-w_{t,a}-w_{t,b})
\intertext{
  By two applications of Lemma~\ref{lem:Xi}, one to the first exponential with $r = \frac{\eta}{w_a} \in [\pm C]$ and $x=B_t-\Delta_a \in [\pm U]$ and one to the second exponential with the analogous substitution for arm $b$, we find
  }
  &~\le~
  w_{t,a}
  \dell{
    1
    + \frac{\eta}{w_{t, a}} (B_t-\Delta_a)
    + \frac{\eta^2}{w_{t, a}^2} \Xi(C, U, \sigma^2)
  }
  \\
  & \quad + w_{t,b} \dell{
    1
    - \frac{\eta}{w_{t, b}} (B_t - \Delta_b)
    + \frac{\eta^2}{w_{t, b}^2} \Xi(C, U, \sigma^2)
  }
  + (1-w_{t,a}-w_{t,b})
  \\
  &~=~
  1
  + \eta (\Delta_b - \Delta_a)
  + \eta^2  \dell{
    \frac{1}{w_{t, a}}
    + \frac{1}{w_{t, b}}
  } \Xi(C, U, \sigma^2)
  \\
  &~\le~
  e^{\eta (\Delta_b - \Delta_a)
    + \eta^2  \dell{
      \frac{1}{w_{t, a}}
      + \frac{1}{w_{t, b}}
    } \Xi(C, U, \sigma^2)
  }.
\end{align*}
Taking the expectation and the logarithm gives the result.
\end{proof}

Similarly, we can bound the exponential moment of our evidence measure $z_{t,a,b}$, which we now do. 
\begin{lemma}\label{lem:MGFbound_z}
  There are constants $\bar \lambda$ and $\bar U$, depending only on $U$ and $\sigma^2$, such that in every round, for $\eta$ with $\sup_a \frac{|\eta|}{w_{t,a}} \le C$ for some $C > 0$,
  \[ \ln \mathbb{E}\left[ e^{\eta z_{t,a,b}} \right] \le {\eta \mathbb{E}[z_{t,a,b}]} + {\eta^2 \left(\frac{1}{w_{t,a}} + \frac{1}{w_{t,b}}\right) \Xi(C, \bar{U}, \sigma^2 \bar\lambda^2)}. \]
    
\end{lemma}
\begin{proof}
    Let $\bar \lambda \coloneq \frac{U}{2 \sigma^2}$. By Lemma~\ref{lem:lambda}, $\lambda_{t,a,b} \le \bar \lambda$ in every round, as the projected gaps are at most $2U$ and $W \le \frac{1}{4}$. Let $\bar U$ be a constant with $|\lambda_{t,a,b}(B_t - \Delta_a) - \alpha_{t,a,b}| \le \bar U$ and $|\lambda_{t,a,b}(B_t - \Delta_b) - \beta_{t,a,b}| \le \bar U$ in every round; it exists since $|B_t - \Delta_a| \le U$ and $\alpha_{t,a,b}, \beta_{t,a,b}$ are bounded by Lemma~\ref{lem:exact alpha beta} applied to the projected gaps. The rest of the proof is similar to that of Lemma~\ref{lem:MGFbound}. We present it here for completeness.
    \begin{align*}
        &\mathbb{E}\left[ e^{ \eta \left( \frac{\indicator_{t,a}}{w_{t,a}}(\lambda_{t,a,b} X_t - \alpha_{t,a,b}) + \frac{\indicator_{t,b}}{w_{t,b}} (-\lambda_{t,a,b} X_t - \beta_{t,a,b}) \right) } \bigg| B_t \right]\\
        &=\mathbb{E}\left[ e^{ \eta \lambda_{t,a,b} \left( \frac{\indicator_{t,a}}{w_{t,a}} -\frac{\indicator_{t,b}}{w_{t,b}} \right) X_t -\eta\left( \frac{\alpha_{t,a,b}\indicator_{t,a}}{w_{t,a}} +  \frac{\beta_{t,a,b}\indicator_{t,b}}{w_{t,b}} \right) } \bigg | B_t \right]\\ 
        &= w_{t,a} e^{-\eta\frac{\alpha_{t,a,b}}{w_{t,a}}} \mathbb{E}\left[ e^{\eta \frac{\lambda_{t,a,b} X_t}{w_{t,a}} } \bigg| B_t, A_t = a \right] + w_{t,b}e^{-\eta\frac{\beta_{t,a,b}}{w_{t,b}}} \mathbb{E}\left[ e^{-\eta \frac{\lambda_{t,a,b X_t}}{w_{t,b}}} \bigg| B_t, A_t = b \right]\\
        &\qquad\qquad + (1-w_{t,a} - w_{t,b})\\
        &\le w_{t,a} e^{\frac{\eta}{w_{t,a}} \left({-\alpha_{t,a,b}} + {\lambda_{t,a,b}} (B_t - \Delta_a)\right)} e^{\frac{\eta^2\sigma^2\lambda^2_{t,a,b}}{2w^2_{t,a}}}\\
        & \qquad \qquad + w_{t,b} e^{\frac{\eta}{w_{t,b}} \left({-\beta_{t,a,b}} - {\lambda_{t,a,b}} (B_t - \Delta_b)\right)} e^{\frac{\eta^2\sigma^2\lambda^2_{t,a,b}}{2w^2_{t,b}}} + (1-w_{t,a} - w_{t,b})\\
        \intertext{
            Again, applying Lemma~\ref{lem:Xi} twice, one to the first exponential with $r = \frac{\eta}{w_a} \in [\pm C]$ and $x=\lambda_{t,a,b}(B_t-\Delta_a) - \alpha_{t,a,b} \in [\pm \bar{U}]$ and variance proxy $\sigma^2 \lambda_{t,a,b}^2 \le \sigma^2 \bar\lambda^2$, using that $\Xi$ is increasing in its last argument, and one to the second exponential with the analogous substitution for arm $b$, we find}\\
        &\le w_{t,a} \left( 1 + \frac{\eta}{w_{t,a}} (\lambda_{t,a,b}(B_t - \Delta_a) - \alpha_{t,a,b}) + \frac{\eta^2}{w^2_{t,a}} \Xi(C, \bar{U}, \sigma^2 \bar\lambda^2) \right) \\
        &\qquad\qquad + w_{t,b} \left( 1 + \frac{\eta}{w_{t,b}} (-\lambda_{t,a,b}(B_t - \Delta_b) - \beta_{t,a,b}) + \frac{\eta^2}{w^2_{t,b}} \Xi(C, \bar{U}, \sigma^2 \bar\lambda^2) \right)\\
        &\qquad\qquad + (1-w_{t,a} - w_{t,b})\\
        &= 1 + \eta \left( \lambda_{t,a,b}(\Delta_b - \Delta_a) - \alpha_{t,a,b} - \beta_{t,a,b} \right) + \eta^2 \left(\frac{1}{w_{t,a}} + \frac{1}{w_{t,b}}\right)\Xi(C, \bar{U}, \sigma^2 \bar\lambda^2)\\
        &\le e^{\eta\left(\lambda_{t,a,b} \Delta_{a,b} - \alpha_{t,a,b} - \beta_{t,a,b} \right)} e^{\eta^2 \left(\frac{1}{w_{t,a}} + \frac{1}{w_{t,b}}\right) \Xi(C, \bar{U}, \sigma^2 \bar\lambda^2)}\\
        &= e^{\eta \mathbb{E}[z_{t,a,b}]}e^{\eta^2 \left(\frac{1}{w_{t,a}} + \frac{1}{w_{t,b}}\right) \Xi(C, \bar{U}, \sigma^2 \bar\lambda^2)}.
    \end{align*}
    Finally, taking the total expectation and the logarithm gives the result.

\end{proof}

\section{\algname{} Sample Complexity Analysis}\label{sec:scproof}

\samplecplx*
At a high level, our proof follows the same broad strategy as the sample-complexity analysis of the TAS algorithm. The main difference lies in the treatment of evidence: TAS recomputes its evidence statistic at each time step, whereas our method accumulates evidence over time. This distinction leads us to work with a different set of concentration events.

\begin{proof}
Here, we derive an upper bound on the stopping time as a function of $\delta$. Hence, to make the dependence of $\tau$ on $\delta$ explicit, we denote the stopping time by $\tau_\delta$. 

Let $\bm \Delta$ be the true gap with best arm $a^*$. For any pair-wise gap vector $\bar{\bm \Delta}$, and all pairs of arms $a\ne b$, let $\lambda_{a,b}(\bar{\bm \Delta})$, $\alpha_{a,b}(\bar{\bm \Delta})$, $\beta_{a,b}(\bar{\bm \Delta})$, and $\bm w(\bar{\bm \Delta})$ denote the optimizers in $D^*(\bar{\bm \Delta})$. For $\epsilon > 0$ (to be chosen small enough, later), define
\[ I_\epsilon := [\bm\Delta - \epsilon, \bm\Delta + \epsilon] = \otimes_{b\in [K]} [\Delta_b - \epsilon, \Delta_b + \epsilon].\]
We will choose $\epsilon$ small enough so that for any $\bar{\bm \Delta}\in I_\epsilon$, we have that $\bar{\Delta}_{a^*,b} > 0$, for all $b\ne a^*$ (so that the best arm $a^*$ for $\bm \Delta$ is also the best arm in $\bar{\bm \Delta}$). Further, since $\bm\Delta$ is fixed, in the rest of this proof, we denote $D^*(\bm \Delta)$ by $D^*$. Note that as $\bm \Delta \in [0,2U]^K$, $\hat{\bm \Delta}_t \in I_\epsilon$ implies $\hat{\bm \Delta}_t \wedge 2U \in I_\epsilon$, so the projected gaps that \algname{} plugs in lie in $I_\epsilon$ whenever $\hat{\bm \Delta}_t$ does.

Then, there exists $\epsilon_1(\epsilon)$ such that tuning for $\bar{\bm \Delta}$ is only $\epsilon_1$ suboptimal compared to tuning for the true $\bm \Delta$ itself, i.e.
\[ \bar{\bm \Delta} \in I_\epsilon
  ~\implies~
  \lambda_{a^*,b}(\bar{\bm \Delta}) \Delta_b - \alpha_{a^*,b}(\bar{\bm \Delta}) - \beta_{a^*,b}(\bar{\bm \Delta}) \ge D^* - \epsilon_1(\epsilon) \quad \forall ~b \neq a^*. \numberthis\label{eq:epsilon1} \]
Note that the existence of $\epsilon_1(\epsilon)$ follows from Lemma~\ref{eq:continuity}: the optimizers depend continuously on $\bar{\bm \Delta}$ and $\Delta_b$ is fixed, so the left-hand side of \eqref{eq:epsilon1} is continuous in $\bar{\bm \Delta}$ and equals $D^*$ at $\bar{\bm \Delta} = \bm \Delta$.
We suppress the dependence of $\epsilon_1$ on $\epsilon$, for simplicity of notation.
For $T  > 0$, and $\epsilon_2 > 0$, define 
\[ \mathcal G_T:= \bigcap\limits_{t=T^\frac{1}{4}}^T \left\{ \hat{\Delta}_t \in I_\epsilon \right\} \bigcap \left\{ \forall b\ne a^* \sum\limits_{t=1}^T z_{t,a^*, b} \ge (T-T^\frac{1}{4}) (D^* - \epsilon_1 - \epsilon_2 ) - T^\frac{1}{4} L \right\}, \numberthis\label{eq:goodset} \]
where $-L$ is a constant lower bound on $\ex_t[z_{t,a^*,b}]$. It exists since the projected gaps, and hence the parameters, are bounded, see the proof of Lemma~\ref{lem:evidenceconc}.
Next, define 
    \[ T_\delta= \min \left\{T:(T-T^\frac{1}{4})(D^*-\epsilon_1-\epsilon_2) -T^\frac{1}{4} L \ge \log\frac{K-1}{\delta}  \right\}. \]
We now show that for any $T\ge T_\delta$, on $\mathcal G_T$, $\tau_\delta \le T$. To see this, fix $T > T_\delta$, and recall from \eqref{eq:stopping.and.recommendation} that on $\mathcal G_T$, since $\hat{\Delta}_t \in I_\epsilon$ for all $t\ge T^\frac{1}{4}$, 
\begin{align*} 
    \tau_\delta 
    &\le \min\left\{ t \ge T^\frac{1}{4}: \min\limits_{b\ne a^*}  \sum\limits_{s=1}^t z_{s,a^*,b} \ge \log\frac{K-1}{\delta} \right\}.
\end{align*}
Using this, 
\begin{align*}
\left\{ \tau_\delta \ge T \right\}\cap \mathcal{G}_T &\subset \left\{ \min\limits_{b\ne a^*}~\sum\limits_{s=1}^T z_{s,a^*,b} \le \ln\frac{K-1}{\delta}\right\} ~\bigcap ~\mathcal G_T\\
&\subset \left\{ (T-T^\frac{1}{4})(D^* - \epsilon_1 - \epsilon_2) - T^\frac{1}{4}L \le \min\limits_{b\ne a^*} \sum\limits_{s=1}^T z_{s,a^*,b} \le \ln\frac{K-1}{\delta}\right\}\\
&= \emptyset, \tag{Since $T > T_\delta$}
\end{align*}
proving that  $\tau_\delta > T_\delta \implies \mathcal G^c_T$. We use this implication in the following:
\begin{align*}
    \mathbb{E}\left[ \tau_\delta \right]
    &= \sum\limits_{T=0}^\infty \mathbb{P}\left(\tau_\delta > T\right)\\
    &\le T_\delta + \sum\limits_{T=T_\delta}^\infty \mathbb{P}\left( \tau_\delta > T \right)\\
    &\le T_\delta + \sum\limits_{T=T_\delta}^\infty \mathbb{P}\left( \mathcal G^c_T \right).
\end{align*}
Using Lemmas~\ref{lem:probgc} and~\ref{lem:tdelta} to bound the two terms in the bound above, dividing both sides by $\log\frac{1}{\delta}$, and taking the limit of $\delta \to 0$, we get
\[ \lim\limits_{\delta\rightarrow 0} ~ \frac{\mathbb{E}[\tau_\delta]}{\log\frac{1}{\delta}} \le \frac{1}{D^* - \epsilon_1 - \epsilon_2}. \]
Finally, since the above bound holds for any $\epsilon >0$ and $\epsilon_2 > 0$, we get the desired bound by optimizing over these free parameters (taking $\epsilon, \epsilon_2 \to 0$).
\end{proof}

The following lemma proves the continuity of the optimizers in $D^*$ as a function of the gap vector.
\begin{lemma}\label{eq:continuity}
  Fix an arm $a^* \in [K]$ and restrict to gap vectors $\bm \Delta \in [0, 2U]^K$ for which $a^*$ is the unique best arm, i.e.\ $\Delta_{a^*} < \Delta_b$ for all $b \neq a^*$. On this set, the maximisers of $D^*(\bm \Delta)$,
  \[
    \bm \Delta ~\mapsto~ \bm w(\bm \Delta)
    \qquad \text{and} \qquad
    \bm \Delta ~\mapsto~ \dell{\lambda_{a^*,b}(\bm \Delta), \alpha_{a^*,b}(\bm \Delta), \beta_{a^*,b}(\bm \Delta)}
    \quad \text{for } b \neq a^*,
  \]
  are continuous in $\bm \Delta$.
\end{lemma}
\begin{proof}
We start by plugging the partial solution to the $D$ problem from Lemma~\ref{lem:exact alpha beta} into the $D^*$ problem. Following the notation of Lemma~\ref{lem:lambda}, this results in
\[
  D^*(\bm \Delta)
  ~=~
  \max_{\bm w \in \triangle_K} \min_{b \neq a^*} \max_{\lambda \ge 0}~
  \lambda \Delta_b - \frac{\lambda^2 \sigma^2}{2 W_{a^*,b}} - h_{2,a^*,b}(\lambda)
  .
\]
This is a convex problem jointly in the weights $(w_a)_{a \in [K]}$ and parameters $(\lambda_{a^*,b})_{b \neq a^*}$ (which we can hoist out of the $\min$). Our $\lambda$ and $\bm w$ are found as its argmax, and we want to see their continuity in $\bm \Delta$. The objective function here is not bounded below (the problematic case is taking weights to $0$, for which both $\frac{1}{W} \to \infty$ and $h_2(\lambda) \to \infty$). Yet by Lemma~\ref{lem:lambda}, we may restrict the domain to $\lambda \in \sbrl{\frac{\Delta W}{\sigma^2 + U^2},\frac{\Delta W}{\sigma^2}}$ without changing the optimum. On that domain, which is compact, the function value \emph{is} bounded below, namely by
\begin{align*}
  &
  \lambda \Delta - \frac{\lambda^2 \sigma^2}{2 W} - h_2(\lambda)
  \\
  &~\ge~
  \lambda \Delta - \frac{\lambda^2 (\sigma^2 + U^2)}{2 W}
  \\
  &~\ge~
  \frac{\Delta W}{\sigma^2} \Delta - \frac{(\frac{\Delta W}{\sigma^2})^2 (\sigma^2 + U^2)}{2 W}
  \\
  &~=~
  \frac{\Delta^2 W}{2 \sigma^2}
  - \frac{\Delta^2 W}{2 \sigma^4} U^2
  \\
  &~\ge~
  -\frac{\Delta^2 U^2}{8 \sigma^4}
\\
  &~\ge~
  -\frac{U^4}{2 \sigma^4}
\end{align*}
Hence, the objective function, being convex and bounded, is continuous. By Berge's theorem, we conclude that $\argmax_{\bm w,\bm\lambda} D^*$ is an upper-hemicontinous (compact, convex, non-empty) set-valued function of $\bm \Delta$.
\end{proof}

\begin{lemma}\label{lem:probgc} For $T > 1$, 
\[ \mathbb{P}(\mathcal G^c_T) \le 2T K e^{- \frac{\epsilon^2 T^\frac{1}{8}}{8 K \Xi(C, U, \sigma^2)}} + K e^{-\frac{\epsilon^2_2 \sqrt{T}}{128 K \Xi(C, \bar{U}, \sigma^2 \bar\lambda^2)}}. \]
\end{lemma}
\begin{proof} Recall from~\eqref{eq:goodset} that
\[ \mathcal G_T:= \underbrace{\bigcap\limits_{t=T^\frac{1}{4}}^T \left\{ \hat{\Delta}_t \in I_\epsilon \right\}}_{:=A} \bigcap \underbrace{\left\{ \forall b\ne a^* \sum\limits_{t=1}^T z_{t,a^*, b} \ge (T-T^\frac{1}{4}) (D^* - \epsilon_1 - \epsilon_2 ) - T^\frac{1}{4} L \right\}}_{:=B}. \]
Using this, 
\begin{align*}
    \mathbb{P}\left( \mathcal G^c_T \right) &= \mathbb{P}(A^c) + \mathbb{P}(B^c \cap A) \\
    &\le \sum\limits_{t=T^\frac{1}{4}}^T\sum\limits_{b}\mathbb{P}\left(  \hat{\Delta}_{t,a^*,b} \notin I_\epsilon \right) \\
    &\qquad \qquad + \sum\limits_{b\ne a^*}\mathbb{P}\left( \left\{\sum\limits_{t=1}^T z_{t,a^*,b} < (T- T^\frac{1}{4})(D^* - \epsilon_1 - \epsilon_2) - T^\frac{1}{4} L \right\} \bigcap A  \right).
\end{align*}
We bound the two terms on the right-hand side above, separately. Using Lemma~\ref{lem:deltaconc}, we get
\begin{align*}
    \sum\limits_{t=T^\frac{1}{4}}^T \sum\limits_{b} \mathbb{P}\left( \hat{\Delta}_{t,a^*,b} \notin I_\epsilon \right) \le 2T K e^{- \frac{\epsilon^2 T^\frac{1}{8}}{8 K \Xi(C, U, \sigma^2)}}.
\end{align*}
Similarly, using Lemma~\ref{lem:evidenceconc}, we get 
\[    \sum\limits_{b\ne a^*}\mathbb{P}\left( \left\{\sum\limits_{t=1}^T z_{t,a^*,b} < (T- T^\frac{1}{4})(D^* - \epsilon_1 - \epsilon_2) - T^\frac{1}{4} L \right\} \bigcap A  \right) \le K e^{-\frac{\epsilon^2_2 \sqrt{T}}{128 K \Xi(C, \bar{U}, \sigma^2 \bar\lambda^2)}}.
\]
We get the desired bound by summing the two bounds.
\end{proof}

\begin{lemma}\label{lem:deltaconc}
    For every pair of arms $a \neq b$, time-step $t \ge 1$, and $\epsilon > 0$ such that $\epsilon/(4 \Xi(C, U, \sigma^2)) \le C$, the following hold:
    \[ \mathbb{P}\left( \hat{\Delta}_{t,a,b} \le \Delta_{a,b} - \epsilon  \right) \le  e^{-\frac{\epsilon^2 \sqrt{t}}{8 K \Xi(C, U, \sigma^2)}}, \]
    and
    \[ \mathbb{P}\left( \hat{\Delta}_{t,a,b} \ge \Delta_{a,b} + \epsilon \right) \le e^{-\frac{\epsilon^2 \sqrt{t}}{8 K \Xi(C, U, \sigma^2)}}, \]
    where $\Xi(C, U, \sigma^2)$ is the constant from Lemma~\ref{lem:MGFbound}.
\end{lemma}
\begin{proof}
Recall, that 
\[ \hat{\Delta}_{t,a,b} = \frac{1}{t} \sum\limits_{s=1}^t \left( \frac{\indicator_{s,a}}{w_{s,a}} - \frac{\indicator_{s,b}}{w_{s,b}} \right)X_s. \]
For any $\eta > 0$, the following hold: 
    \begin{align*}
    \mathbb{P}\left( \hat{\Delta}_{t,a,b} \ge \Delta_{a,b} + \epsilon \right) 
    &\le \mathbb{P}\left( \eta \sum\limits_{s=1}^t \left( \frac{\indicator_{s,a}}{w_{s,a}} - \frac{\indicator_{s,b}}{w_{s,b}} \right)X_s \ge t\eta \Delta_{a,b} + t\eta \epsilon \right) \\
    &\le \mathbb{E}\left[ e^{\eta\sum\limits_{s=1}^t \left( \frac{\indicator_{s,a}}{w_{s,a}} - \frac{\indicator_{s,b}}{w_{s,b}} \right) X_s} \right] e^{-t\eta \Delta_{a,b} - t\eta\epsilon}\\
    &\overset{(a)}{\le} e^{\sum\limits_{s=1}^t \left( \eta \Delta_{a,b} + \eta^2\left( \frac{1}{w_{s,a}} + \frac{1}{w_{s,b}}  \right)\Xi(C, U, \sigma^2)   \right)}e^{-t\eta \Delta_{a,b} - t\eta\epsilon}\\
    &\le e^{\sum\limits_{s=1}^t \left( \eta \Delta_{a,b} + 2\eta^2\Xi(C, U, \sigma^2)  K \sqrt{s}  \right)}e^{-t\eta \Delta_{a,b} - t\eta\epsilon} \tag{$w_{s,a} \ge \frac{1}{K \sqrt{s}}$, $\forall a$}\\
    &\le e^{2\eta^2\Xi(C, U, \sigma^2)  K t\sqrt{t}}e^{- t\eta\epsilon},
    \end{align*}
    where the inequality $(a)$ follows from Lemma~\ref{lem:MGFbound} (assuming that $\sup_a \eta/w_{s,a} \le C$) and the constant $\Xi(C, U, \sigma^2)$ is as in the Lemma. We will later see that the optimized $\eta$ indeed satisfies this condition for small $\epsilon$, so that the assumption in Lemma~\ref{lem:MGFbound} is satisfied, and this step is legal. 

    Finally, we get the desired bound by optimizing the right-hand side in the above inequality with respect to $\eta \ge 0$ (setting $\eta = \eta^* := {\epsilon}/(4 \Xi(C, U, \sigma^2) K \sqrt{t})$). Now, recall that by forced exploration, $w_{s,a} \ge \frac{1}{K \sqrt{s}}$, and hence, $\eta^* / w_{s,a} \le \epsilon/(4 \Xi(C, U, \sigma^2))$. Choosing $\epsilon$ small enough guarantees that this bound is at most $C$, proving the result. 
    The other bound follows symmetrically.
\end{proof}

\begin{lemma}\label{lem:evidenceconc} For $T > 1$, $\epsilon > 0$ and 
\[  A:= \bigcap\limits_{t=T^\frac{1}{4}}^T \left\{ \hat{\Delta}_t \in I_\epsilon \right\} ,\]
we have
\[\mathbb{P}\left( \left\{\sum\limits_{t=1}^T z_{t,a^*,b} < (T- T^\frac{1}{4})(D^* - \epsilon_1 - \epsilon_2) - T^\frac{1}{4} L \right\}  \bigcap A \right) \le e^{-\frac{\epsilon^2_2 \sqrt{T}}{128 K \Xi(C, \bar{U}, \sigma^2 \bar\lambda^2)}}, \]
where the constants $C$, $\bar{U}$, and $\Xi(C, \bar{U}, \sigma^2 \bar\lambda^2)$ are as in Lemma~\ref{lem:MGFbound_z}.
\end{lemma}
\begin{proof}
    For $t\ge 1$, recall from \eqref{eq:z} that
    \[ z_{t,a^*,b} = \frac{\indicator_{t,a^*}}{w_{t,a^*}} (\lambda_{t,a^*,b} X_t - \alpha_{t,a^*,b}) - \frac{\indicator_{t,b}}{w_{t,b}}(\lambda_{t,a^*,b}X_t + \beta_{t,a^*,b}), \]
    where the parameters $\lambda_{t,a^*,b}$, $\alpha_{t,a^*,b}$, and $\beta_{t,a^*,b}$ are the optimizers in \eqref{eq:params} tuned for the projected gaps $\hat{\bm \Delta}_{t-1} \wedge 2U$. For the bandit with true gap $\bm \Delta$,
    \[ \mathbb{E}\left[ z_{t,a^*,b} \right] = \lambda_{t,a^*,b} \Delta_b - \alpha_{t,a^*,b} - \beta_{t,a^*,b}. \]
    Clearly, on the event $A$, for $t \in [T^\frac{1}{4}, T]$, $\mathbb{E}[ z_{t,a^*,b}] \ge  D^* - \epsilon_1$. This follows from~\eqref{eq:epsilon1}, together with the observation that $\epsilon_1$ can be taken small enough to accommodate the forced exploration added in \eqref{eq:w}. Moreover, for $t < T^\frac{1}{4}$, we also have $\mathbb{E}[z_{t,a^*,b}] \ge -L$, for a constant $L$. This follows since $\Delta_b$, $\lambda_{t,a^*,b}$, $\alpha_{t,a^*,b}$, and $\beta_{t,a^*,b}$ are bounded (by Assumption~\ref{ass:main} and Lemmas~\ref{lem:lambda} and~\ref{lem:exact alpha beta}, the latter applied to the projected gaps in $[0, 2U]$; this is where the projection in \eqref{eq:params} is needed). Using these, for any $\eta > 0$ we get the following inequalities.
    \begin{align*}
        &\mathbb{P}\left( \left\{\sum\limits_{t=1}^T z_{t,a^*,b} < (T- T^\frac{1}{4})(D^* - \epsilon_1 - \epsilon_2) - T^\frac{1}{4} L \right\}  \bigcap A \right) \\
        &\le \mathbb{P}\left( \sum\limits_{t=1}^T z_{t,a^*,b}  < \sum\limits_{t=1}^T \mathbb{E}[z_{t,a^*,b}] - \epsilon_2 (T- T^\frac{1}{4}) \right)\\
        &= \mathbb{P}\left( \sum_{t=1}^T \left( \frac{\indicator_{t,a^*}}{w_{t,a^*}}(\lambda_{t,a^*,b} X_t - \alpha_{t,a^*,b}) + \frac{\indicator_{t,b}}{w_{t,b}} (-\lambda_{t,a^*,b} X_t - \beta_{t,a^*,b}) \right) \right. \\
        &\qquad \qquad \left. < \sum_{t=1}^T \mathbb{E}[z_{t,a^*,b}] - \epsilon_2 (T-T^\frac{1}{4})  \right)\\
        &\le \mathbb{E}\left[ e^{ -\eta \sum\limits_{t=1}^T \left( \frac{\indicator_{t,a^*}}{w_{t,a^*}}(\lambda_{t,a^*,b} X_t - \alpha_{t,a^*,b}) + \frac{\indicator_{t,b}}{w_{t,b}} (-\lambda_{t,a^*,b} X_t - \beta_{t,a^*,b}) \right) } \right] e^{\eta \sum_{t=1}^T \mathbb{E}[z_{t,a^*,b}] - \eta\epsilon_2(T-T^\frac{1}{4})}\\
        &\overset{(\star)}{\le} e^{\eta^2\sum\limits_{t=1}^T \left(\frac{1}{w_{t,a^*}} + \frac{1}{w_{t,b}} \right) \Xi(C, \bar{U}, \sigma^2 \bar\lambda^2)} e^{ -\eta\epsilon_2(T-T^\frac{1}{4})}\\
        &\le e^{2\eta^2 K T \sqrt{T}\Xi(C, \bar{U}, \sigma^2 \bar\lambda^2)} e^{ -\eta\epsilon_2(T-T^\frac{1}{4})} \tag{Since $w_{t,a} \ge \frac{1}{K \sqrt{t}}$ for all $a$},
    \end{align*}
    where $(\star)$ follows from Lemma~\ref{lem:MGFbound_z} (assuming $\sup_a |\eta|/w_{s,a} \le C$) and the constants $C$, $\bar{U}$ and $\Xi(C, \bar{U}, \sigma^2 \bar\lambda^2)$ are as in Lemma~\ref{lem:MGFbound_z}.

    Finally, we get the desired bound by optimizing the right-hand side in the inequality above with respect to $\eta \ge 0$ (setting $\eta = \eta^* = \epsilon_2(T-T^\frac{1}{4})/(4 K T \sqrt{T} \Xi(C, \bar{U}, \sigma^2 \bar\lambda^2))$), and observing that $T-T^\frac{1}{4} \ge \frac{1}{4} T$ for $T > 1$.
\end{proof}

\begin{lemma}\label{lem:tdelta}
    For $A>0$, $B>0$, and threshold $\mathfrak{L} > 0$, let 
    \[ T_{\mathfrak{L}}= \min \left\{T:(T-T^\frac{1}{4})A -T^\frac{1}{4} B \ge \mathfrak{L} \right\}. \]
    Then, 
    \[ T_{\mathfrak{L}} \le \frac{\mathfrak{L}}{A} + o\left(\mathfrak{L}\right).\]
\end{lemma}
\begin{proof}
By definition, $T_\delta$ is the smallest integer $T$ such that 
\begin{equation}
\label{eq:basic_ineq}
AT - (A+B) T^{1/4} \ge \mathfrak{L}.
\end{equation}
Fix a constant $u>0$ (depending only on $A$ and $B$, and independent of $\mathfrak{L}$) and consider
\[
T := \frac{\mathfrak{L}}{A} + u \mathfrak{L}^{1/4}.
\]
For all sufficiently large $\mathfrak{L}$, using concavity of the map $x\to x^\frac{1}{4}$,
\[
T^{1/4}
\le \left(\frac{\mathfrak{L}}{A}\right)^{1/4} + (u \mathfrak{L}^{1/4})^{1/4}
= A^{-1/4} \mathfrak{L}^{1/4} + u^{1/4} \mathfrak{L}^{1/16}.
\]
Substituting into the left-hand side of \eqref{eq:basic_ineq}, we obtain
\begin{align*}
AT - (A+B)T^{1/4}
&= \mathfrak{L} + Au \mathfrak{L}^{1/4} - (A+B)T^{1/4} \\
&\ge \mathfrak{L} + Au \mathfrak{L}^{1/4}
- (A+B)\big(A^{-1/4} \mathfrak{L}^{1/4} + u^{1/4} \mathfrak{L}^{1/16}\big).
\end{align*}
Choosing $u$ sufficiently large so that
\[
Au \ge 2 (A+B) A^{-1/4},
\]
the coefficient of $\mathfrak{L}^{1/4}$ is positive, and the remaining
$\mathfrak{L}^{1/16}$ term is negligible. Hence, for all sufficiently large $\mathfrak{L}$,
\[
AT - (A+B)T^{1/4} \ge \mathfrak{L},
\]
which shows that
\[
T_{\mathfrak{L}} \le \frac{\mathfrak{L}}{A} + u \mathfrak{L}^{1/4} = \frac{\mathfrak{L}}{A} + o(\mathfrak{L}),
\]
as claimed.
\end{proof}

\section{Variance of evidence}\label{apx:worst.case.variance}
We have
\begin{align*}
  &\ex_t\sbrcl{
    z_{t,a,b}^2
  }{B_t}
  \\
  &~=~
  \ex_t\sbrcl{
    \dell{
      \frac{\indicator_{t, a}}{w_{t, a}} \cdot \dell{
        \lambda_{t,a,b} X_t - \alpha_{t,a,b}
      }
      + \frac{\indicator_{t, b}}{w_{t, b}} \cdot \dell{
        -\lambda_{t,a,b} X_t - \beta_{t,a,b}
      }}^2
  }{B_t}
  \\
  &~=~
  \frac{1}{w_{t, a}} \ex_t\sbrcl{
     \dell{
        \lambda_{t,a,b} X_t - \alpha_{t,a,b}
      }^2
  }{A_t = a, B_t}
  + \frac{1}{w_{t, b}} \ex_t\sbrcl{
    \dell{
      -\lambda_{t,a,b} X_t - \beta_{t,a,b}
    }^2
  }{A_t = b, B_t}
  \\
  &~\le~
  \frac{
    (\lambda_{t,a,b}  (B_t - \Delta_a) - \alpha_{t,a,b})^2+\lambda_{t,a,b} ^2 \sigma ^2
  }{w_{t, a}}
  + \frac{
    (-\lambda_{t,a,b}  (B_t-\Delta_b) - \beta_{t,a,b})^2+\lambda_{t,a,b} ^2 \sigma ^2
  }{w_{t, b}}
\end{align*}
To get a worst-case bound, this convex function of $B_t$ needs to be maximised over the valid interval of $B_t \in [-U+\Delta_{\max}, U]$. It suffices to check the end-points and take the max. Finally, the variance is the second moment minus the square of the mean, which is $\ex_t[z_{t,a,b}] = \lambda_{t,a,b} \Delta_b - \alpha_{t,a,b} - \beta_{t,a,b}$.

\section{Computation}\label{sec:compute}
We start from the observation that we can write the $D^*(\bm \Delta)$ problem as
\[
  \max_{\bm w \in \triangle_K} \min_{b \neq a^*} w_{a^*} g_b\dell{\frac{w_b}{w_{a^*}}}
\]
where
\[
  g_b(r)
  ~=~
  \max_{\zeta}~
  \frac{r}{1+r} \dell{
    \zeta \frac{\Delta_b}{U}
    - \frac{\zeta^2 \sigma^2}{2 U^2}
  }
  +  h \dell{\frac{1}{1+r} \zeta}
  + r h \dell{
    \frac{r}{1+r} \zeta
  }
  - (1+r) h \dell{
    \zeta
  }
\]
Following \citet[Section~2.2]{garivier16optimalBAI}, this means we are looking for the largest possible level $C$ to which all $g_b$ can be equalised, for weights summing to one. That is, we are after
\[
  \max_C ~ \frac{C}{1 + \sum_{b \neq a^*} g_b^{-1}(C)}
  .
\]
Solving that concave maximisation problem amounts to cancelling the derivative, i.e.\ solving
\[
  \sum_{b \neq a^*} \setl{
    C (g_b^{-1})'(C)
    - g_b^{-1}(C)
  }
  ~=~
  1
  .
\]
Now let $r_b(C) = g_b^{-1}(C)$ and let $\zeta_b(r)$ be the maximiser in $g_b(r)$. Then we are looking for $C$ such that
\[
  \sum_{b \neq a^*} \setl{
    \frac{
      C
    }{
      g_b'(r_b(C))
    }
    - r_b(C)
  }
  ~=~
  1
\]
By the rule for the derivative of an inverse, we have
\[
  g_b'(r)
  ~=~
  \frac{
    \zeta \frac{\Delta_b}{U}
    - \frac{\zeta^2 \sigma^2}{2 U^2}
    - \zeta h' \dell{\frac{1}{1+r} \zeta}
    + r \zeta h' \dell{
      \frac{r}{1+r} \zeta
    }
  }{(1+r)^2}
  + h \dell{
    \frac{r}{1+r} \zeta
  }
  - h \dell{
    \zeta
  }
\]
With that, we have the pleasingly symmetric
\[
  \frac{
    C
  }{
    g_b'(r_b(C))
  }
  - r_b(C)
  ~=~
  \frac{
    \frac{
      r^2 \dell{
        \zeta \frac{\Delta_b}{U}
        - \frac{\zeta^2 \sigma^2}{2 U^2}
      }
      + r \zeta \dell{
         h' \dell{\frac{1}{1+r} \zeta}
        - r h' \dell{
          \frac{r}{1+r} \zeta
        }
      }
    }{(1+r)^2}
    + h \dell{\frac{1}{1+r} \zeta}
    - h \dell{
      \zeta
    }
  }{
    \frac{
      \dell{
        \zeta \frac{\Delta_b}{U}
        - \frac{\zeta^2 \sigma^2}{2 U^2}
      }
      - \zeta \dell{
         h' \dell{\frac{1}{1+r} \zeta}
        - r h' \dell{
          \frac{r}{1+r} \zeta
        }
      }
  }{(1+r)^2}
  + h \dell{
    \frac{r}{1+r} \zeta
  }
  - h \dell{
    \zeta
  }
  }
\]
All in all, this reveals that $\argmax D^*(\bm \Delta)$ can be computed with three nested bisections. The outer one optimises over the scalar $C$. Then for each arm $b \neq a^*$, the middle one computes the inverse $r_b(C)$, and within that, the innermost evaluates $g_b(r)$.

The advantage of this method is it reaches machine precision in $O(K)$ work. We may compare this to first-order iterative schemes (including gradient descent), where the run-time per iteration is also $O(K)$, but where the required number of iterations increases with the desired accuracy. We may also contrast it with second-order iterative schemes (including Newton's method), where the number of iterations is small, but the  run-time per iteration is $O(K^2)$.

\cite{purex.games} show that it is possible to get the best of both worlds: one can run one iteration of an iterative scheme per round, while updating the estimates between round, and obtain convergence guarantees using concentration and online learning regret bounds. We leave developing such schemes for our setting as an interesting future endeavour.

\section{Numerical Experiments - Additional Figures} \label{app:numerical}

\begin{figure}[H]
  \centering
  \begin{tabular}{cc}
    \includegraphics[width=.485\textwidth]{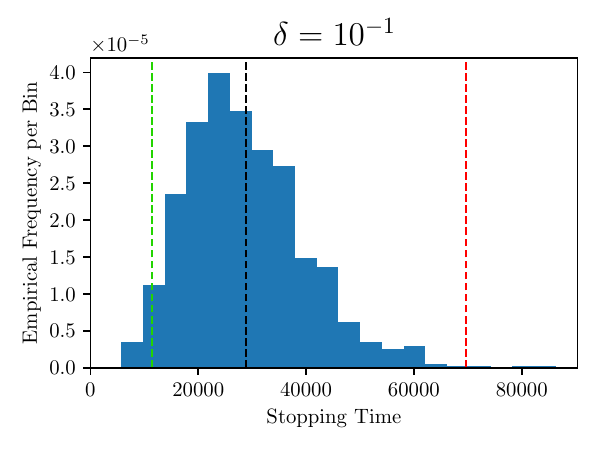}&
    \includegraphics[width=.485\textwidth]{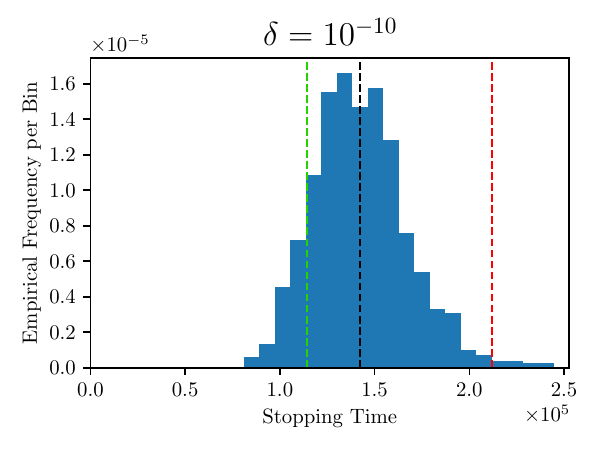}\\
    \includegraphics[width=.485\textwidth]{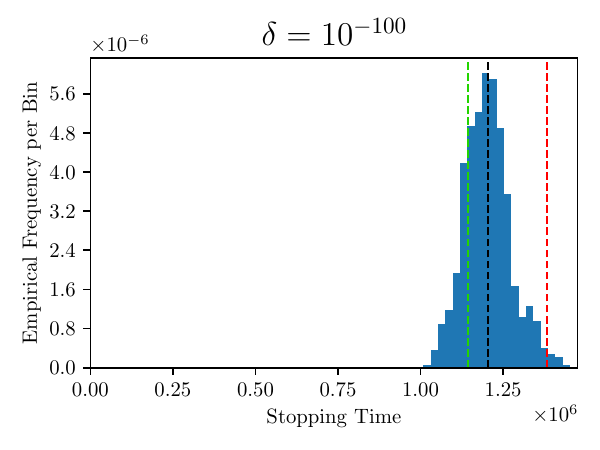}&
    \includegraphics[width=.485\textwidth]{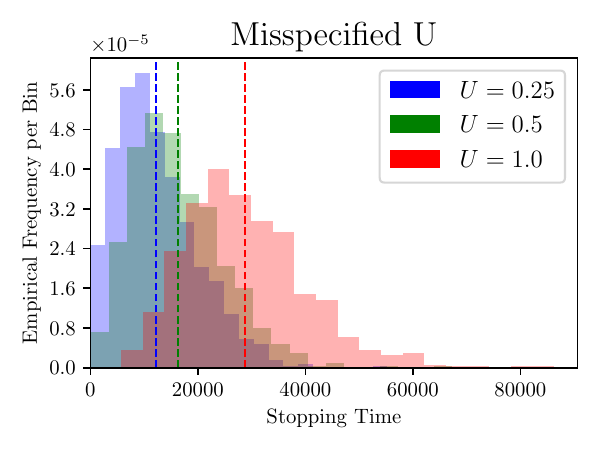}
  \end{tabular}
  \caption{Histograms of stopping times in $1000$ runs. The first three figures use $U=1$ and vary $\delta$, the dashed lines are (green) the expected stopping time as predicted by Theorem~\ref{thm:sample complexity}, (black) the empirical average stopping time, and (red) the stopping time as predicted by Lemma~\ref{lem:oracle} (as also spelled out in Table~\ref{tab:stopping times}). The last figure fixes $\delta = 1/10$ and mis-specifies $U = 1, \frac{1}{2}, \frac{1}{4}$. Here the dashes show the average stopping times. $U = \frac{1}{4}$ admits 4 mistakes, all other runs were perfect.
  }\label{fig:ism_e-1_delta}
\end{figure}

\begin{table}[H]
\centering
\begin{tabular}{lllllllll}
  $a$ & $b$ & $\Delta_b$ & $\lambda_{t,a,b}$ & $\alpha_{t,a,b}$ & $\beta_{t,a,b}$ & Drift $\ex_t[z_{t,a,b}]$ & $\var_t[z_{t,a,b}]$ bound \\
  \hline
  1 & 2 & 0.05 & 0.00805581 & 0.000100929 & 0.000100477 & 0.000201384 & 0.00038665 \\
  1 & 3 & 0.07 & 0.0057545 & 0.000104116 & 0.000097315 & 0.000201384 & 0.000374914 \\
  1 & 4 & 0.1 & 0.00402871 & 0.000105826 & 0.000095661 & 0.000201384 & 0.0003842
\end{tabular}
\caption{Optimised parameters for the pairs $a \neq b$ involving the best arm $a=a^*=1$. The drifts are perfectly equalised. The worst-case variances are computed as described in Appendix~\ref{apx:worst.case.variance}. They are not optimised to be similar, but they turn out to be.}\label{tab:values}
\end{table}

\end{document}

%% file: preamble.tex
\usepackage{graphicx} %

\usepackage{amsmath} %
\usepackage{amsthm} %
\usepackage{float}

\usepackage{algorithm}
\usepackage{algorithmic}

\usepackage{dsfont}
\usepackage{xcolor}
\usepackage{thmtools, thm-restate}
\usepackage{amsfonts}
\usepackage{amssymb} %
\usepackage{bm}
\usepackage{bbm}

\providecommand{\DeclareBoldMathCommand}[2]{%
  \expandafter\newcommand\csname #1\endcsname{\bm{#2}}%
}

\usepackage[shortlabels]{enumitem}

\usepackage{etoolbox}
\usepackage{thm-restate}
\usepackage{regexpatch}

\usepackage{nicefrac}
\usepackage{mathrsfs}
\usepackage{comment}

\usepackage[normalem]{ulem}

\makeatletter
\xpatchcmd\thmt@restatable{%
\csname #2\@xa\endcsname\ifx\@nx#1\@nx\else[{#1}]\fi
}{%
\ifthmt@thisistheone
\csname #2\@xa\endcsname\ifx\@nx#1\@nx\else[{#1}]\fi
\else
\csname #2\@xa\endcsname[{\textsc{restated}}]
\fi}{}{}
\makeatother

\newtheorem{theorem}{Theorem}
\newtheorem{lemma}[theorem]{Lemma}

\newtheorem{definition}[theorem]{Definition}

\newcommand*{\Acal}{\mathcal{A}}
\newcommand*{\Bcal}{\mathcal{B}}

\newcommand*{\Mcal}{\mathcal{M}}

\newcommand*{\reals}{\mathbb{R}}
\newcommand*{\nats}{\mathbb{N}}

\DeclareBoldMathCommand{\0}{0}
\DeclareBoldMathCommand{\1}{1}

\DeclareBoldMathCommand{\a}{a}
\DeclareBoldMathCommand{\A}{A}
\DeclareBoldMathCommand{\b}{b}
\DeclareBoldMathCommand{\B}{B}
\DeclareBoldMathCommand{\bell}{\ell}
\DeclareBoldMathCommand{\C}{C}
\DeclareBoldMathCommand{\D}{D}
\DeclareBoldMathCommand{\e}{e}
\DeclareBoldMathCommand{\Ebold}{E}
\DeclareBoldMathCommand{\g}{g}
\DeclareBoldMathCommand{\h}{h}
\DeclareBoldMathCommand{\H}{H}
\DeclareBoldMathCommand{\I}{I}
\DeclareBoldMathCommand{\L}{L}
\DeclareBoldMathCommand{\p}{p}
\DeclareBoldMathCommand{\P}{P}
\DeclareBoldMathCommand{\Phat}{\hat{P}}
\DeclareBoldMathCommand{\q}{q}
\DeclareBoldMathCommand{\S}{S}
\DeclareBoldMathCommand{\u}{u}
\DeclareBoldMathCommand{\v}{v}
\DeclareBoldMathCommand{\V}{V}
\DeclareBoldMathCommand{\w}{w}
\DeclareBoldMathCommand{\x}{x}
\DeclareBoldMathCommand{\y}{y}
\DeclareBoldMathCommand{\z}{z}

\DeclareBoldMathCommand{\mubold}{\mu}

\DeclareBoldMathCommand{\zetabold}{\zeta}
\DeclareBoldMathCommand{\xibold}{\xi}

\DeclareMathOperator{\E}{\mathbb{E}}
\DeclareMathOperator{\pr}{\mathbb{P}}

\DeclareMathOperator*{\Et}{\mathbb{E}_t}

\DeclareMathOperator*{\argmax}{arg\,max}
\DeclareMathOperator*{\argmin}{arg\,min}

\newcommand\numberthis{\addtocounter{equation}{1}\tag{\theequation}}
\newcounter{relctr} %
\everydisplay\expandafter{\the\everydisplay\setcounter{relctr}{0}} %

\AtBeginDocument{} %

\usepackage{mathtools}

\DeclarePairedDelimiterX{\inner}[2]{\langle}{\rangle}{#1, #2}
\DeclarePairedDelimiterX{\norm}[1]{\|}{\|}{#1}

\newcommand*{\Prob}{\mathbb{P}}

\usepackage{xparse}
\NewDocumentCommand\mset{ g }{
  \IfNoValueTF{#1}{\Bcal_{K, B}}{\Bcal_{K, #1}}
}

\newtheorem{assumption}{Assumption}

\makeatletter
\newcounter{learningparadigm}

\makeatother

\makeatletter
\newcommand{\given}{%
  \@ifstar
    {\mathrel{}\middle|\mathrel{}}%
    {\mid}%
}
\makeatother

\newcommand*{\sums}{\sum_{s = 1}^{t}}

\newcommand{\indicator}{\mathbb{I}}

\DeclareBoldMathCommand{\boldzero}{0}

\DeclareBoldMathCommand{\boldsigma}{\Sigma}

\DeclareBoldMathCommand{\lossmat}{\Theta}
\DeclareBoldMathCommand{\cocov}{\Sigma}

\DeclareBoldMathCommand{\Icalbold}{\mathcal{I}}

\DeclareBoldMathCommand{\Phit}{\Phi}
\DeclareBoldMathCommand{\Psit}{\Psi}
\DeclareBoldMathCommand{\Upsilont}{\Upsilon}

\DeclarePairedDelimiterX{\infdivx}[2]{(}{)}{%
  #1\;\delimsize\|\;#2%
}
\newcommand{\KL}{\textnormal{KL}\infdivx}

\newcommand{\kl}{\textnormal{kl}\infdivx}

\DeclareBoldMathCommand{\occ}{q}

\AtBeginDocument{%
  \let\l\relax
  \let\r\relax
  \newcommand{\l}{\left}%
  \newcommand{\r}{\right}%
}